\documentclass[10pt,journal,compsoc]{IEEEtran}

\usepackage{graphicx,amsmath, amssymb,lineno}
\usepackage{subfigure, graphicx}
\usepackage{amsmath,amssymb,amsthm}
\usepackage{algorithmic}
\usepackage{algorithm}
\usepackage{tabularx}
\usepackage{multirow}
\usepackage{makecell}
\newtheorem{lemma}{Definition}
\newtheorem{theorem}{Proposition}
\ifCLASSOPTIONcompsoc
  \usepackage[nocompress]{cite}
\else
  \usepackage{cite}
\fi

\ifCLASSINFOpdf

\else

\fi

\begin{document}

\title{Secure Classification With Augmented Features}

\author{Chenping~Hou,~Ling-Li Zeng, Dewen~Hu
\thanks{Chenping Hou is with the College of Science, National University of Defense Technology, Changsha, 410073, Hunan, China. E-mail: hcpnudt@hotmail.com.}
\thanks{Ling-Li Zeng and Dewen Hu is with the College of Mechatronics and Automation, National University of Defense Technology, Changsha, 410073, Hunan, China. Email: lingl.zeng@gmail.com, dwhu@nudt.edu.cn}
}

\markboth{Secure Classification With Augmented Features}%
{Hou \MakeLowercase{\textit{et al.}}: Secure Classification With Augmented Features}

\IEEEtitleabstractindextext{
\begin{abstract}
With the evolution of data collection ways, it is possible to produce abundant data described by multiple feature sets. Previous studies show that including more features does not necessarily bring positive effect. How to prevent the augmented features worsening classification performance is crucial but rarely studied. In this paper, we study this challenging problem by proposing a \emph{secure} classification approach, \emph{whose accuracy is never degenerated when exploiting augmented features}. We propose two ways to achieve the security of our method named as SEcure Classification (SEC). Firstly, to leverage augmented features, we learn various types of classifiers and adapt them by employing a specially designed robust loss. It provides various candidate classifiers to meet the following assumption of security operation. Secondly, we integrate all candidate classifiers by approximately maximizing the performance improvement. Under a mild assumption, the integrated classifier has theoretical security guarantee. Several new optimization methods have been developed to accommodate the problems with proved convergence. Besides evaluating SEC on 16 data sets, we also apply SEC in the application of diagnostic classification of schizophrenia since it has vast application potentiality. Experimental results demonstrate the effectiveness of SEC in both tackling security problem and discriminating schizophrenic patients from healthy controls.
\end{abstract}

}

\maketitle

\IEEEdisplaynontitleabstractindextext

\section{Introduction}

Classification is a traditional and essential research area in many fields, such as machine learning, data mining and image processing. Classical methods, such as KNN, Naive Bayes, Boosting and SVM all assume that the data are represented by a unique feature set \cite{Bishop:2006:PRM}. With the advent of data collection ways, it is possible to collect abundant data described by multiple distinct feature sets \cite{xu2013MVLsurvey, ZhaoXXS17, ChenZSX12, HouNTY17}. For example, images can be characterized by different descriptors, such as SIFT, HOG etc. A piece of news can be described by text, image, hyperlink etc. In medical image analysis, the functional MRI (fMRI) image can be characterized by the image feature descriptors or the region-to-region functional connectivity features.

In the literature, multi-view learning is a standard learning paradigm to manipulate data with distinct descriptions \cite{xu2013MVLsurvey, ZhaoXXS17, ChenZSX12, HouNTY17}. The heterogeneous features are regarded as a particular view in multi-view learning. Rather than requiring that all the examples have been comprehensively described by a single view, it might be better to exploit the connections between multiple views to improve the performance. Up to now, there are a number of multi-view learning algorithms and they have been combined with or applied to various computer vision and intelligent system problems. One can refer to \cite{xu2013MVLsurvey} and \cite{ZhaoXXS17} for a comprehensive review of multi-view learning.

As to multi-view classification, there are many prominent researches. Roughly, they can be categorized into three groups. The first group is Multiple Kernel Learning (MKL) \cite{jmlr/GonenA11}. Its original intention is to avoid the determination of kernels in the classical Support Vector Machine (SVM) classifier. If we learn a kernel in each view, MKL is naturally suitable for multi-view classification. There are a lot of efforts have been conducted on MKL. For example, Lanckriet et al. have formulated MKL as the dual of a SVM problem and solved it by a semi-definite programming \cite{icml/LanckrietCBGJ02}. Bach et al. have proposed a block norm regularized MKL method \cite{nips/BachTJ04}, Zien and Ong have learned a common feature space for all the different classes and extended MKL for multiclass classification (McMKL) \cite{icml/ZienO07}. The second group is based on subspace learning. It tries to learn a subspace across different views for classification. Farquhar et al. have integrated KCCA and SVM into a single optimization formulation and proposed SVM-2K \cite{nips/FarquharHMSS05}. Diethe et al. have extended traditional Fisher Discriminant Analysis (FDA) by deriving a regularized two-view equivalent of FDA and cast it to multiple views FDA (MvFDA) \cite{Diethe2008Multiview}. Kan et al. have proposed Multi-view Discriminant Analysis (MvDA), which seeks for a single discriminant common space for multiple views in a non-pairwise manner by jointly learning multiple view-specific linear transforms \cite{pami/KanSZLC16}. The third group is based on regression. For instance, the Subspace Co-regularized Multi-View (SCMV) learning method \cite{Guo2012Cross} developed by Guo et al. utilizes the regression model in the projected common subspace of all views. Zheng et al. have proposed the Multi-view Low-Rank Regression (MLRR) \cite{Zheng2015}, which combines regression model on each view by adding a low rank constraint across different views.

Although these methods perform excellently in the related fields, they all focus on the utilization or extension of traditional methods for multi-view classification. There are still some fundamental issues should be aware. Before going into the details, we would like to introduce the setting of one experiment for demonstration and the results are shown in Table \ref{table_sigle_double_compare}. We have conducted experiments on data characterized in Table \ref{table_DataDetail}. The data named as AD1add2 means that we take the first views of AD data on hand and the second view of AD is regarded as new coming features, and similarly for AD1add3. Four representative multi-view classification methods from the above mentioned three categories, i.e., MLRR, MvFDA, MvDA and McMKL, have been employed. For comparison, we take the first view of each data as the input of these methods and the classification accuracy is recorded in the line \emph{Single} in Table \ref{table_sigle_double_compare}. Correspondingly, if both views of each data sets have been considered, the results are listed in the line \emph{Double}. We take 30\% points for training and the rest for testing.

\begin{table}[!t]
\caption{Accuracy comparison between single view and double views.}
\label{table_sigle_double_compare}
\centering
\vskip -0.1in
{
\begin{tabular}{c | c || c c c }
\hline
  Methods             &  View No.     & Ionoshpere        & AD1add2          & AD1add3 \\
\hline
\multirow{2}{*}{MLRR} & Single        &  0.7696           &  0.8381          & 0.8381      \\
                      & Double        &  0.9098           &  0.7652          & 0.7234      \\
\hline
\multirow{2}{*}{MvFDA}& Single        &  0.9235           &  0.8658          & 0.8658      \\
                      & Double        &  0.9292           &  0.8349          & 0.9173      \\
\hline
\multirow{2}{*}{MvDA} & Single        &  0.8531           &  0.8545          & 0.8545      \\
                      & Double        &  0.9888           &  0.8714          & 0.8867      \\
\hline
\multirow{2}{*}{McMKL}& Single        &  0.9061           &  0.9015          & 0.9015      \\
                      & Double        &  0.9718           &  0.9597          & 0.9455      \\
\hline
\end{tabular}}
\vskip -0.1in
\end{table}

As seen from above results, traditional methods cannot guarantee that the utilization of new coming features could always improve the performance. For instance, as shown in Table \ref{table_sigle_double_compare}, the performances of MLRR on AD1add2 and AD1add3 degrade with more views. Similarly, MvFDA also performs worse on AD1add2 when the second view is added. One important reason may be that, whether the utilization of new coming features will facilitate the performance is closely related to the type of classifiers. For example, as seen from the results in Table \ref{table_sigle_double_compare}, MvFDA performs better than MvDA on data AD1add3 whereas the relationship is opposite on AD1add2. The same phenomenon can also be observed from the results of MLRR on Ionoshpere and AD1add2 data sets. Thus, how to design a \emph{secure} classifier, which guarantees that the performance does not become worse with more features, no matter which type of classifier is employed, is also worthy of investigation.

In this paper, we consider the question how to address the above major and fundamental issue in multi-view learning. To the best of our knowledge, these questions have not been thoroughly studied. Specifically, there are totally two stages of our approach, i.e., Adaption stage (A-stage) and Integration stage (I-stage). In A-stage, we train different types of classifiers on the data in the first view and adapt all these classifier by employing data in the new view. In I-stage, we derive a final classifier based on adapted classifiers. With weak assumption, it is guaranteed to perform no worse than all the classifiers trained on the first view. In other words, we design a secure classier.

We present a SEC (Secure Classifier) learning framework. In A-stage, without domain knowledge about which type of classifier performs best on data in the first view, we employ several popular classifiers as the baseline classifiers $\{f_1, f_2, \cdots, f_m\}$. When the new view comes, we use new view data to adapt these classifiers to $\{g_1, g_2, \cdots, g_m\}$, by only requiring the classifiers and their predictions on the first view. The robust multi-class capped $\ell_1$ norm loss has been employed for adaption to reduce the adaption bias. The balance between two views is also determined automatically, without additional parameter. In I-stage, SEC proposes to maximize the performance gain of $g \triangleq \ell(g_1, g_2, \cdots, g_m)$ against all the classifiers in $\{f_1, f_2, \cdots, f_m\}$. The resultant new formulation is a quadratic constrained linear problem, which is a well-defined problem and can be solved by several state-of-art optimization tools. We can also show that SEC is provably secure and have already achieved the maximal performance gain approximately with a mild assumption. Experimental results on a broad range of datasets validate that SEC clearly leverages the new coming features.

Our contributions are summarized as follows.

(1) We propose the SEC approach to solve this crucial, but rarely studied problem. As far as we know, this is the first research about the secure classifier learning problem with new coming features.

(2) We propose to tackle this problem in two-stage way. In A-stage, we propose a new classifier adaption approach. In I-stage, we present a new learning paradigm to derive a secure classifier. We develop an efficient algorithm to address the optimization problem and give theoretical analyses about the security.

(3) A byproduct of our approach is that instead of accessing the original data, only the classifiers and their predictions on the first view are needed. Thus, it is a privacy preserving approach as mentioned in \cite{cikm/YeZMJZ15}.

(4) Except for systematical evaluation on 16 data sets from a broad range, we also utilize our method in the diagnostic classification of schizophrenia with MRI, which has vast application potentiality. All the experimental results indicate that our algorithm outperforms other compared algorithms in almost all cases. It provides a new perspective to investigate this problem and some useful guidance for the researchers in this field.

The rest of the paper is organized as follows. Section \ref{sec_TTFC} will formulate this problem and we propose the SEC approach and some theoretical results to guarantee the security. The optimization of SEC is introduced in Section \ref{sec_optimization}. Experimental results on benchmark datasets are displayed in Section \ref{sec_exp}, together with the application to diagnostic classification of schizophrenia with MRI. Finally, we conclude this paper in Section \ref{sec_con}.

\section{The SEC Algorithm}
\label{sec_TTFC}

In our work, we focus on the learning problem in two views. Concretely, we have data with one view on hand already and another view being coming. Our main purpose is utilizing the new-view data to train a secure classifier, whose performance is never statistically worse than that of the classifiers trained on the data on hand. Before going into the details, let us introduce some notations at first.

\subsection{Notations and Definition}

Assume that $\{\mathbf{x}_i^{(1)} \in \mathbb{R}^{d_1} \}_{i=1}^n$ and $\{\mathbf{x}_i^{(2)} \in \mathbb{R}^{d_2} \}_{i=1}^n$ are the training data points from two views respectively. Here $n$ is the number of training points, $d_1$ and $d_2$ are the dimensionality of data in view 1 and view 2. The superscript indicates the view index. Correspondingly, denote $\{\mathbf{x}_i^{(1)} \in \mathbb{R}^{d_1} \}_{i=n+1}^{n+t}$ and $\{\mathbf{x}_i^{(2)} \in \mathbb{R}^{d_2} \}_{i=n+1}^{n+t}$ as the testing points and $t$ is the number of testing points. For simplicity, the data matrix is denoted as
$\mathbf{X}_{tr}^{(1)} = [\mathbf{x}_1^{(1)},\mathbf{x}_2^{(1)},\cdots,\mathbf{x}_n^{(1)}]$,
$\mathbf{X}_{tr}^{(2)} = [\mathbf{x}_1^{(2)},\mathbf{x}_2^{(2)},\cdots,\mathbf{x}_n^{(2)}]$,
$\mathbf{X}_{te}^{(1)} = [\mathbf{x}_{n+1}^{(1)},\mathbf{x}_{n+2}^{(1)},\cdots,\mathbf{x}_{n+t}^{(1)}]$ and
$\mathbf{X}_{te}^{(2)} = [\mathbf{x}_{n+1}^{(2)},\mathbf{x}_{n+2}^{(2)},\cdots,\mathbf{x}_{n+t}^{(2)}]$.
Denote the total number of classes as $c$. We employ the one-vs-rest binary coding scheme to encode the class labels. In other words, the label vector of training data $\mathbf{x}_i^{(1)}$ and $\mathbf{x}_i^{(2)}$  is represented by $\mathbf{y}_i \in \{-1, 1\}^{c\times 1}$, such that $y_i(j)=1$ iff $\mathbf{x}_i^{(1)}$ and $\mathbf{x}_i^{(2)}$ belong to the $j$-th category and $y_i(j)=-1$ otherwise. Then, the label matrix is $\mathbf{Y}_{tr}=[\mathbf{y}_1,\mathbf{y}_2, \cdots,\mathbf{y}_n] \in \{-1,1\}^{c\times n}$.

\begin{lemma}
A classifier $f$ is secure with respect to $\{f_j\}_{j=1}^m$ means that the performance of $f$ is never statistically significantly worse
than the best classification results of $\{f_j\}_{j=1}^m$.
\end{lemma}

\subsection{SEC Approach}

We investigate the multi-view classifier learning problem from new perspectives, i.e., secure. It is difficult to use traditional approaches to achieve this goal directly. In our paper, we tackle it in the following two-stage way.

~~1. In A-stage, we learn different types of baseline classifiers, i.e., $\{f_1, f_2, \cdots, f_m\}$, on data with the first view and then adapt them to $\{g_1, g_2,\cdot, g_m\}$ by utilizing new-view data. They are various candidate classifiers to meet the assumption of the following security operation in I-stage.

~~2. In I-stage, with security guarantee, we integrate $ \{ g_1, g_2, \cdots, g_m\}$ as $g = \ell(g_1, g_2, \cdots, g_m)$ by approximately maximizing the improvement in performance.

\subsubsection{A-stage}

In applications, different classifiers have different model assumptions. It is difficult to design a unique classifier which always achieves the best performances on all kinds of data sets. To get a secure classifier, we first train different types of classifiers on the data with the first view, i.e. $\mathbf{X}_{tr}^{(1)}$ and $\mathbf{Y}$, and adapt them using the new-view data.

Concretely, assume that the classifiers are denoted as $f_1, f_2, \cdots, f_m$, where $m$ is the number of classifiers. For example, $f_j$ could be classification function of KNN, Regression model, Naive Bayes, Boosting methods, SVM or any kind of classifiers. When the new-view data, i.e., $\mathbf{X}_{tr}^{(2)}$, comes, the most direct way is employing traditional multi-view classification approaches as mentioned above. Nevertheless, it may lead to at least two problems in using this strategy. (1) Most of these methods require that all the data, i.e., $\mathbf{X}_{tr}^{(1)}$, $\mathbf{X}_{tr}^{(2)}$ and $\mathbf{Y}$ are already accessible. In many applications, it may violate the privacy preserving requirement \cite{cikm/YeZMJZ15}. For example, in image retrieval, researchers can extract their own features of images as well as get the retrieval results from other popular search engines, such as Google. They cannot obtain the image features used by Google. (2) Retraining the multi-view classifier is time and storage consuming. In our paper, we will adapt the pre-trained existing classifiers, i.e., $f_1,f_2,\cdots, f_m$, by only using their predictions, without accessing the data $\mathbf{X}_{tr}^{(1)}$.

\begin{figure}[!t]
\centering
\includegraphics[width=0.4\textwidth]{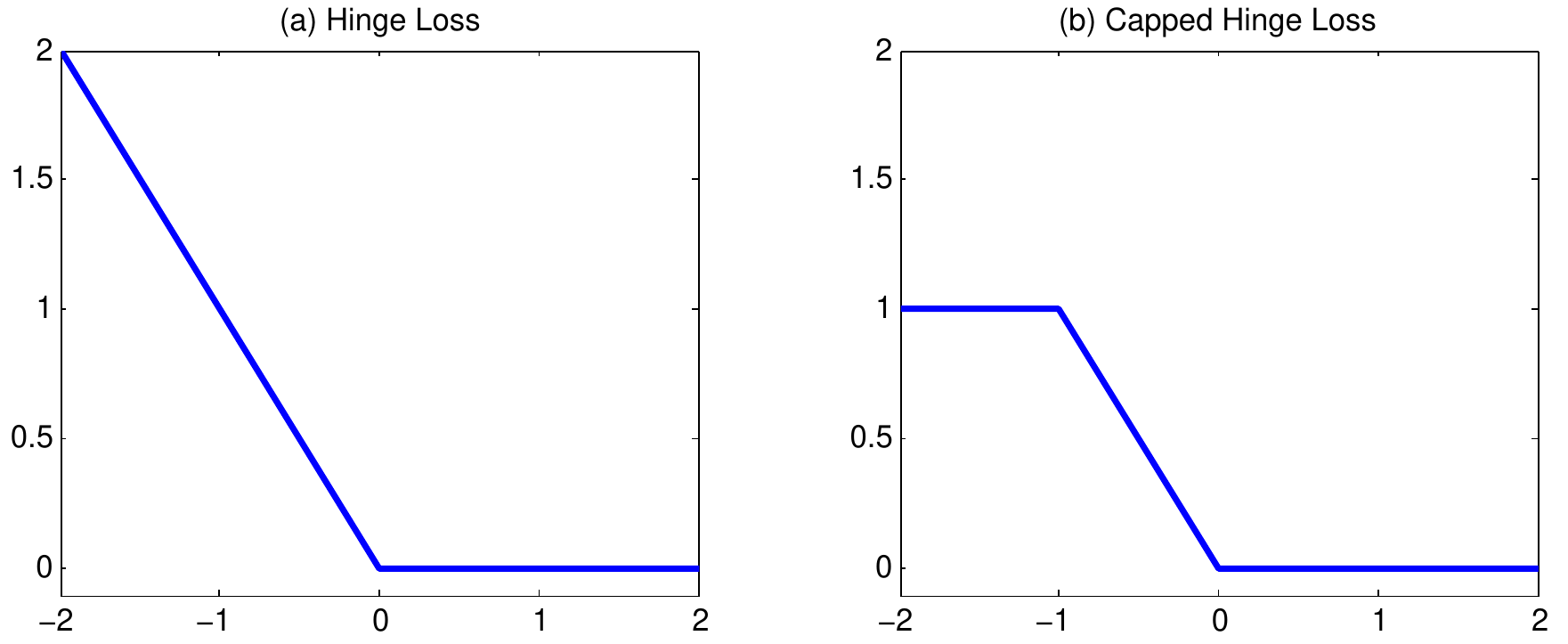}
\caption{The plots of different loss functions. The abscissa is the value of $y_i g_j(\mathbf{x}_i^{(1)},\mathbf{x}_i^{(2)})-1$.}
\label{fig_loss}
\vskip -0.2in
\end{figure}

Assume that the prediction of $f_j$ ($j=1, 2, \cdots, m$) on the first view's training data $\mathbf{x}_i^{(1)}$ ($i=1, 2, \cdots, n$) is $f_j(\mathbf{x}_i^{(1)})$. Due to the privacy preserving requirement or the missing of data in the first view, we can only get $f_j(\mathbf{x}_i^{(1)})$, i.e., the prediction of $f_j$ on $\mathbf{x}_i^{(1)}$. In this case, since the classifier $f_j$ is trained on $\mathbf{X}_{tr}^{(1)}$, $f_j(\mathbf{x}_i^{(1)})$ could approximate $\mathbf{y}_i$. When the new representation, i.e., $\mathbf{x}_i^{(2)}$, comes, we should train a new classifier to approximate the training error between $f_j(\mathbf{x}_i^{(1)})$ and $\mathbf{y}_i$.

Considering that the kernel trick is flexible and SVM is one of the most popular classifiers, we assume that the classifier trained on $\mathbf{x}_i^{(2)}$ is $\mathbf{W}^T \Phi(\mathbf{x}_i^{(2)})+\mathbf{b}$, where $\mathbf{W} \in \mathbb{R}^{D\times c}$, with $D$ the dimensionality of high-dimensional mapping space, is the transformation matrix, $\Phi(\cdot)$ is the mapping function and $\mathbf{b}\in \mathbb{R}^{c\times 1}$ is the bias. Now, the adapted classifier on both $\mathbf{x}_i^{(1)}$ and $\mathbf{x}_i^{(2)}$ is defined as follows
\begin{equation}
\label{eq1}
\small{
\begin{split}
g_j(\mathbf{x}_i^{(1)},\mathbf{x}_i^{(2)}) \triangleq \lambda_1 f_j(\mathbf{x}_i^{(1)})+\lambda_2 (\mathbf{W}^T \Phi(\mathbf{x}_i^{(2)})+\mathbf{b}).
\end{split}}
\end{equation}
where $\lambda_1$ and $\lambda_2$ are two balance parameters, which will be learned automatically.

For the convenience of presentation, we first focus on the problem in two-class scenario. Since our goal is to utilize $\mathbf{x}_i^{(2)}$ to obtain better classifier, the most popular loss metric is the hinge loss. The definition is
\begin{equation}
\label{eq2}
\begin{split}
 &\ell_h(g_j|\mathbf{x}_i^{(1)},\mathbf{x}_i^{(2)}) \\
=&\max(1-y_i g_j(\mathbf{x}_i^{(1)},\mathbf{x}_i^{(2)}), 0) \\
=&\max(1- \lambda_1 y_i f_j(\mathbf{x}_i^{(1)})- \lambda_2 y_i (\mathbf{w}^T \Phi(\mathbf{x}_i^{(2)})+b), 0),
\end{split}
\end{equation}
where $y_i \in \{-1,1\}$ is the class label of both $\mathbf{x}_i^{(1)}$ and $\mathbf{x}_i^{(2)}$, $\mathbf{w}\in \mathbb{R}^D$ is the transformation vector. $f_j(\mathbf{x}_i^{(1)})$ is the decision value and $b$ is the bias.

We would like to explain why we employ the hinge loss in essence. As seen from Eq. (\ref{eq2}) and the left plane in Fig. \ref{fig_loss}, in fact, we use $y_i(\mathbf{w}^T \Phi(\mathbf{x}_i^{(2)})+b)$ to approximate the error $1- \lambda_1 y_i f_j(\mathbf{x}_i^{(1)})$. If this value is small, it means that we can classify the $i$-th training data by merely using $\mathbf{x}_i^{(1)}$. In this scenario, the hinge loss between them is small (less than one) and we should not pay much attention to them. Otherwise, if $\mathbf{x}_i^{(1)}$ is categorized into the wrong class, the hinge loss is large and we need to use $\mathbf{x}_i^{(2)}$ to reduce the approximation error. Briefly, the utilization of hinge loss has roughly categorized $\mathbf{x}_i^{(1)}$ into correctly or wrongly classified points and the hinge loss requires us to emphasize on the wrongly classified points.

Whereas in traditional hinge loss, if the data point is not correctly classified, the loss could be infinite. Since we use $\mathbf{w}^T \Phi(\mathbf{x}_i^{(2)})+b$ to compensate the disagreement between $f_j(\mathbf{x}_i^{(1)})$ and $y_i$. It will mislead the computation of $\mathbf{w}$ and $b$ to fit the largest disagreement. Recall the essence of classification, we should focus on minimizing the classification error, not the classification loss. Thus, we will trunk the larger loss and apply the capped hinge loss shown in the right plane in Fig. \ref{fig_loss} for adaption.
\begin{equation}
\label{eq3}
\begin{split}
 \ell_{ch}(g_j|\mathbf{x}_i^{(1)},\mathbf{x}_i^{(2)})
= \min( \max(1-y_i g_j(\mathbf{x}_i^{(1)},\mathbf{x}_i^{(2)}), 0),1).
\end{split}
\end{equation}

To extend this loss into multi-class scenario, we use the similar strategy as in \cite{XiangNMPZ12, NieWH17}. Concretely, the two-class hinge loss in Eq. (\ref{eq2}) can be reformulated as follows.
\begin{equation}
\label{eq4}
\begin{split}
 \ell_{h}(g_j|\mathbf{x}_i^{(1)},\mathbf{x}_i^{(2)}) =\min_{m_i\geq 0} | g_j(\mathbf{x}_i^{(1)},\mathbf{x}_i^{(2)}) - y_i-y_i m_i|,
\end{split}
\end{equation}
where $m_i \in \mathbb{R}$ is a slack variable to encode the loss.

Inspired by this reformulation, it is natural to extend the binary capped hinge loss in Eq. (\ref{eq3}) into the multi-class scenario as follows.
\begin{equation}
\label{eq5}
\small{
\begin{split}
 & \ell_{mch}(g_j|\mathbf{x}_i^{(1)},\mathbf{x}_i^{(2)})=\\
& \min(\min_{\mathbf{m}_i\geq 0} \| g_j(\mathbf{x}_i^{(1)},\mathbf{x}_i^{(2)}) - \mathbf{y}_i-\mathbf{y}_i\circ \mathbf{m}_i\|_2,1)= \\
& \min(\min_{\mathbf{m}_i\geq 0} \| \lambda_1 f_j(\mathbf{x}_i^{(1)})+\lambda_2 (\mathbf{W}^T \Phi(\mathbf{x}_i^{(2)})+\mathbf{b}) - \mathbf{y}_i-\mathbf{y}_i \circ \mathbf{m}_i\|_2,1),
\end{split}}
\end{equation}
where $\mathbf{m}_i\in \mathbb{R}^{c\times1}$ and $\mathbf{m}_i\geq 0$ means that all the elements of $\mathbf{m}_i$ are non-negative. $\|\cdot\|_2$ means the 2-norm of a vector and $\circ$ is the pairwise product of two vectors.

By summarizing all the training error and adding a regularizer, the objective function of adapting $f_j$ can be formulated as
\begin{equation}
\label{eq6}
\small{
\begin{split}
& \arg \min_{\mathbf{W},\mathbf{b}, \mathbf{M} \geq 0, \lambda_1,\lambda_2} \lambda \|\mathbf{W}\|_{\rm{F}}^{2} + \sum_{i=1}^n \min \\
&\left(\| \lambda_1 f_j(\mathbf{x}_i^{(1)})+\lambda_2 (\mathbf{W}^T \Phi(\mathbf{x}_i^{(2)})+\mathbf{b}) - \mathbf{y}_i-\mathbf{y}_i \circ \mathbf{m}_i\|_2 ,1 \right).
\end{split}}
\end{equation}
Here, $\|\cdot\|_{\rm{F}}$ is the F-norm of a matrix. $\lambda$ is the pre-defined parameter for regularizer. $\mathbf{M} \in \mathbb{R}^{c\times n}$ is a matrix with the $i$-th column as $\mathbf{m}_i$.

There are several points should be highlighted here. (1) Although the formulation in Eq. (\ref{eq6}) seems complicated, we will propose an effective strategy to solve it. (2) There is an implicit function $\Phi$ in this formulation. We will show how to tackle it by kernel trick. (3) There is only one explicit balance parameter $\lambda$. We will tune it by cross validation.

\subsubsection{I-stage}

In A-stage, we have adapted all the on hand classifiers $\{ f_j\}_{j=1}^m$ to their adaption version $\{ g_j\}_{j=1}^m$ shown in Eq. (\ref{eq1}), by incorporating the data in the new coming view. As seen from the results in Table \ref{table_accuracy}, although this kind of adaption could improve the performances of original classifiers in most cases, it cannot guarantee the security. For example, the adaption of Naive Bayes classifier (AdNaiveBayes) performs even worse than the results in the first column in Table \ref{table_accuracy}, which is the best classification results of $\{ f_j\}_{j=1}^m$.

Assume that the prediction of $f_j$ on testing data $\{\mathbf{x}_i^{(1)} \in \mathbb{R}^{d_1} \}_{i=n+1}^{n+t}$ is $f_j(\mathbf{x}_i^{(1)})$. For convenience, we change the definition of label vector from -1 vs 1 to 0 vs 1. The corresponding predicted label vector is denoted as $\hat{\mathbf{y}}_i^{(j)} \in \{0,1\}^{c\times1}$. Denote $\hat{\mathbf{Y}}^{(j)} = [\hat{\mathbf{y}}_{n+1}^{(j)},\hat{\mathbf{y}}_{n+2}^{(j)}, \cdots,\hat{\mathbf{y}}_{n+t}^{(j)}] \in \{0,1\}^{c \times t}$. Correspondingly, assume the prediction of $g_j$ on testing data $\{\mathbf{x}_i^{(1)}, \mathbf{x}_i^{(2)} \}_{i=n+1}^{n+t}$ is $g_j(\mathbf{x}_i^{(1)}, \mathbf{x}_i^{(2)})$. The predicted label vector is denoted as $\bar{\mathbf{y}}_i^{(j)} \in \{0,1\}^{c\times1}$. Denote $\bar{\mathbf{Y}}^{(j)} = [\bar{\mathbf{y}}_{n+1}^{(j)},\bar{\mathbf{y}}_{n+2}^{(j)}, \cdots,\bar{\mathbf{y}}_{n+t}^{(j)}] \in \{0,1\}^{c \times t}$. Denote $\mathbf{Y}^{*} = [\mathbf{y}_{n+1},\mathbf{y}_{n+2}, \cdots,\mathbf{y}_{n+t}] \in \{0,1\}^{c \times t}$ as the ground truth of testing labels.

Note that if $\hat{\mathbf{Y}}^{(j)}$ and $\mathbf{Y}^{*}$ are the 0-1 label matrix, the squared F-norm between them could measure the prediction error of $\mathbf{Y}^{(j)}$. In other words, we have
\begin{equation}
\label{eq7}
\small{
\begin{split}
\rm{Acc}(\hat{\mathbf{Y}}^{(j)}) = 1 - \frac{\|\hat{\mathbf{Y}}^{(j)}- \mathbf{Y}^{*}\|_{\rm{F}}^2}{2t}.
\end{split}}
\end{equation}
Here $\rm{Acc}(\hat{\mathbf{Y}}^{(j)})$ indicates the classification accuracy computed by $\hat{\mathbf{Y}}^{(j)}$.

To guarantee the security, we want to derive $\mathbf{Y} \in \{0,1\}^{c \times t}$, whose classification accuracy is no smaller than all $\rm{Acc}(\hat{\mathbf{Y}}^{(j)})$ for $j=1,2\cdots,m$. In other words, we aim to derive $\mathbf{Y}$, which performs never worse than the largest $\rm{Acc}(\hat{\mathbf{Y}}^{(j)})$. This requirement indicates the following constraints.
\begin{equation}
\label{eq8}
\small{
\begin{split}
\|\mathbf{Y}- \mathbf{Y}^{*}\|_{\rm{F}}^2 \leq \min_j \|\hat{\mathbf{Y}}^{(j)}- \mathbf{Y}^{*}\|_{\rm{F}}^2.
\end{split}}
\end{equation}

Since $\mathbf{Y}^{*}$ is unknown, we cannot obtain the optimal $\mathbf{Y}$ by solving the problem in Eq. (\ref{eq8}) directly. Nevertheless, we have adapted $f_j$ to $g_j$, it is commonly recognized that $g_j$ performs better than $f_j$. In other words, compared with $f_j$, $g_j$ is closer to $\mathbf{Y}^{*}$. This phenomenon can also be observed since $f_j$ is a special case of $g_j$ when $\lambda_1=1$ and $\lambda_2=0$. Based on this observation, we replace $\mathbf{Y}^{*}$ by $\bar{\mathbf{Y}}^{(k)}$ as a surrogate. Since we do not know which $\bar{\mathbf{Y}}^{(k)}$ is the best, the only way is adding all the constraints as follows.
\begin{equation}
\label{eq9}
\small{
\begin{split}
\|\mathbf{Y}- \bar{\mathbf{Y}}^{(k)}\|_{\rm{F}}^2 \leq \min_j \|\hat{\mathbf{Y}}^{(j)}- \bar{\mathbf{Y}}^{(k)}\|_{\rm{F}}^2, ~\forall~k=1,2,\cdots,m.
\end{split}}
\end{equation}

To obtain the optimal solution, TIFC optimizes the following problem.
\begin{equation}
\label{eq10}
\small{
\begin{split}
          & \arg\max_{\mathbf{Y}, \epsilon}~\epsilon\\
\rm{s.t.~} &\|\mathbf{Y}- \bar{\mathbf{Y}}^{(k)}\|_{\rm{F}}^2 = \min_j \|\hat{\mathbf{Y}}^{(j)}- \bar{\mathbf{Y}}^{(k)}\|_{\rm{F}}^2-\epsilon, ~\forall~k=1,2,\cdots, m\\
          & \epsilon \geq 0 \\
          & \mathbf{Y} \in \{0,1\}^{c \times t} \rm{~is~a~label~matrix.}
\end{split}}
\end{equation}

Note that the optimal solution to Eq. (\ref{eq10}) cannot guarantee the secure constraints in Eq. (\ref{eq8}). We will show that with some mild constraints, TIFC is provable secure.

\subsection{Why and When the Proposal Works}

Before going into the detail analysis, as can be observed from Eq. (\ref{eq10}), the distance of $\|\mathbf{Y}- \mathbf{Y}^{*}\|_{\rm{F}}^2$ should be smaller than $\min_j \|\hat{\mathbf{Y}}^{(j)}- \mathbf{Y}^{*}\|_{\rm{F}}^2$ if $\mathbf{Y}^{*} \in \{\bar{\mathbf{Y}}^{(k)}\}_{k=1}^m$. Such an observation motivates us to get the following results.

\begin{theorem} \label{proposition1}
Assume that $\mathbf{Y}$ is the optimal solution to Eq. (\ref{eq10}). $\|\mathbf{Y}- \mathbf{Y}^{*}\|_{\rm{F}}^2 \leq \min_j \|\hat{\mathbf{Y}}^{(j)}- \mathbf{Y}^{*}\|_{\rm{F}}^2$ if the ground truth label matrix $\mathbf{Y}^{*} \in \{\bar{\mathbf{Y}}^{(k)}\}_{k=1}^m$.
\end{theorem}

Proof of this proposition is direct and we would like to omit it. Since the assumption $\mathbf{Y}^{*} \in \{\bar{\mathbf{Y}}^{(k)}\}_{k=1}^m$ is too strict, we would like to relax it and get another proposition.

\begin{theorem} \label{proposition2}
Assume that $\mathbf{Y}$ is the optimal solution to Eq. (\ref{eq10}). If there exists a $\bar{\mathbf{Y}}^{(k)}$, such that
\begin{equation}
\label{eq11}
\small{
\begin{split}
\rm{Tr} ((\hat{\mathbf{Y}}^{(j)}-\mathbf{Y})^\top ( \mathbf{Y}^{*}-\bar{\mathbf{Y}}^{(k)} )) \leq 0,~\forall~j=1,2,\cdots, m.
\end{split}}
\end{equation}
Then, $\|\mathbf{Y}- \mathbf{Y}^{*}\|_{\rm{F}}^2 \leq \min_j \| \hat{\mathbf{Y}}^{(j)}- \mathbf{Y}^{*}\|_{\rm{F}}^2$.
\end{theorem}

\begin{proof}
Without loss of generality, assume that $\bar{\mathbf{Y}}^{(1)}$ satisfies the inequality in Eq. (\ref{eq11}). Since $\mathbf{Y}$ is the optimal solution to Eq. (\ref{eq10}), we have
\begin{equation}
\label{eq12}
\small{
\begin{split}
\|\mathbf{Y}- \bar{\mathbf{Y}}^{(1)}\|_{\rm{F}}^2 \leq \|\hat{\mathbf{Y}}^{(j)}- \bar{\mathbf{Y}}^{(1)}\|_{\rm{F}}^2, ~\forall~j=1,2,\cdots, m.
\end{split}}
\end{equation}
Then, $\forall~j=1, 2, \cdots, m$,
\begin{equation}
\label{eq13}
\small{
\begin{split}
\|\mathbf{Y}- \mathbf{Y}^{*} + \mathbf{Y}^{*}-\bar{\mathbf{Y}}^{(1)}\|_{\rm{F}}^2 \leq \|\hat{\mathbf{Y}}^{(j)}-\mathbf{Y}^{*}+\mathbf{Y}^{*}- \bar{\mathbf{Y}}^{(1)}\|_{\rm{F}}^2.
\end{split}}
\end{equation}

After expanding both side of the above inequality and making some deductions, $\forall~j=1, 2, \cdots, m$, we get
\begin{equation*}
\label{eq14}
\small{
\begin{split}
     &\|\mathbf{Y}- \mathbf{Y}^{*}\|_{\rm{F}}^2-\|\hat{\mathbf{Y}}^{(j)}-\mathbf{Y}^{*}\|_{\rm{F}}^2 \leq \\
 &2\rm{Tr}((\hat{\mathbf{Y}}^{(j)}-\mathbf{Y}^{*})^\top (\mathbf{Y}^{*}- \bar{\mathbf{Y}}^{(1)}))-
2\rm{Tr}((\mathbf{Y}- \mathbf{Y}^{*})^\top (\mathbf{Y}^{*}-\bar{\mathbf{Y}}^{(1)} ))\\
& = 2\rm{Tr}((\hat{\mathbf{Y}}^{(j)}-\mathbf{Y})^\top (\mathbf{Y}^{*}- \bar{\mathbf{Y}}^{(1)})) \leq 0.
\end{split}}
\end{equation*}
The last inequality holds due to the assumption in Eq. (\ref{eq11}).

By minimizing both sides of the above inequality w.r.t. $j$, we get $\|\mathbf{Y}- \mathbf{Y}^{*}\|_{\rm{F}}^2 \leq \min_j \| \hat{\mathbf{Y}}^{(j)}- \mathbf{Y}^{*}\|_{\rm{F}}^2$.
\end{proof}

Proposition \ref{proposition2} shows that the proposed SEC method is provably secure when we have at least one adaption classifier, which satisfies the condition in Eq. (\ref{eq11}). We will verify this proposition via numerical results in Section \ref{sec_exp}. Note that this is just a sufficient rather than necessary condition for SEC. In other words, SEC may also work when the assumption in Eq. (\ref{eq11}) violates.

\begin{figure}[!t]
\centering
\includegraphics[width=0.4\textwidth]{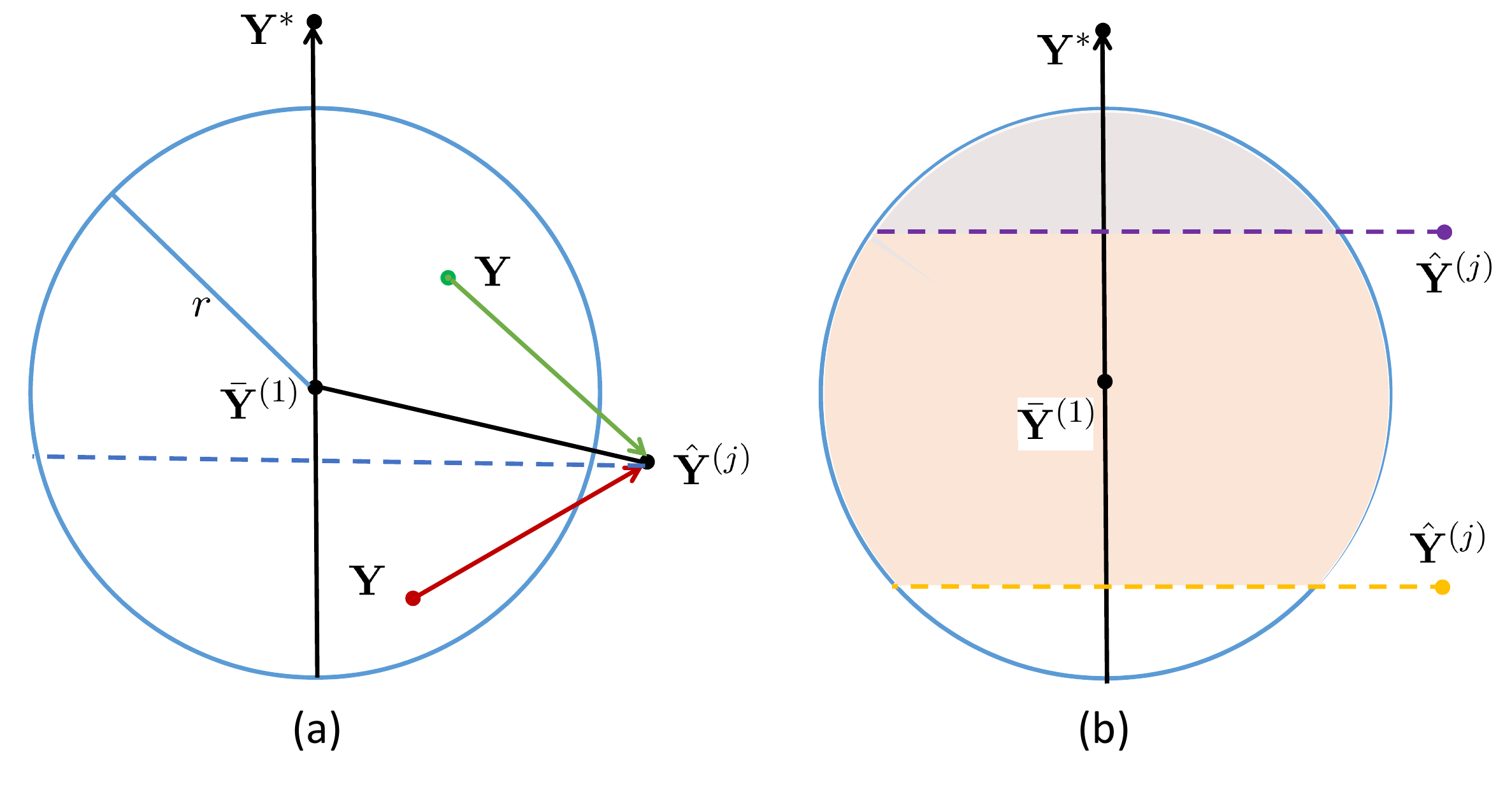}
\caption{Illustration of the intuition of the condition in Eq. (\ref{eq11}) via the inner product viewpoint. Intuitively, the green point $\mathbf{Y}$ is secure.}
\label{fig_illustration}
\vskip -0.15in
\end{figure}

There is an intuitive explanation about the condition in Eq. (\ref{eq11}). Similar to above deduction, assume that $\bar{\mathbf{Y}}^{(1)}$ satisfies the inequality in Eq. (\ref{eq11}). As seen from Fig. \ref{fig_illustration}(a), the first constraint in objective function of SEC in Eq. (\ref{eq10}) indicates that all the feasible solution $\mathbf{Y}$ is located within a circle, whose centre is $\bar{\mathbf{Y}}^{(1)}$ and radius is  $r \triangleq \sqrt{\min_j \|\hat{\mathbf{Y}}^{(j)}- \bar{\mathbf{Y}}^{(1)}\|_{\rm{F}}^2-\epsilon}$. According to the definition of $r$, $\hat{\mathbf{Y}}^{(j)}$ is often located outside of the circle. If we reshape the matrix to vector, the trace operation is just the inner product between the reshaped vectors. Recall that the sign of inner product dominates the angle between two vectors, we can regard that the condition in Eq. (\ref{eq11}) reflects the angle between $\hat{\mathbf{Y}}^{(j)}-\mathbf{Y}$ and $ \mathbf{Y}^{*}-\bar{\mathbf{Y}}^{(k)}$. As illustrated in Fig. \ref{fig_illustration}, if $\mathbf{Y}$ is located in the red point, the angle between $\hat{\mathbf{Y}}^{(j)}-\mathbf{Y}$ and $ \mathbf{Y}^{*}-\bar{\mathbf{Y}}^{(k)}$ is less than 90 degrees and thus, the condition in in Eq. (\ref{eq11}) violates. Otherwise, if $\mathbf{Y}$ is located in the green point, the condition in Eq. (\ref{eq11}) satisfies. Interestingly, when $\mathbf{Y}$ is located in the green point, it is closer to $\mathbf{Y}^{*}$ than that of $\hat{\mathbf{Y}}^{(j)}$. Thus, it is securer than $\hat{\mathbf{Y}}^{(j)}$. Otherwise, when $\mathbf{Y}$ is located in the red point, the inner product is positive and it is far from $\mathbf{Y}^{*}$. In fact, the feasible solution $\mathbf{Y}$, which satisfies the condition in Eq. (\ref{eq11}) is located within the upper circle to the dashed line.

More interestingly, as seen from Fig. \ref{fig_illustration}(b), if the adaption result ($\bar{\mathbf{Y}}^{(1)}$) is better than the original result ($\hat{\mathbf{Y}}^{(j)}$), the feasible region will be the orange area in Fig. \ref{fig_illustration}(b). On the contrary, if the adaption is useless, the feasible region is the gray area, which is much smaller than the orange area. In other words, the reason why we improve the performance of original classifiers with augmented features is to provide well performed candidate classifiers, which can give us a large opportunity to find a feasible solution that satisfies the condition in Eq. (\ref{eq11}).

\subsection{Relation to Other Approaches}

Security is a general problem of many disciplines of machine learning. There are only some researches in semi-supervised learning \cite{icml/LiZ11, pami/LiZ15, aaai/LiKZ16, aaai/LiZZ17}, since using the unlabeled data is not usually helpful. Concretely, these methods focus on the safeness of using unlabeled data. For example, Safe Semi-Supervised SVM (S4VM) classifier is investigated in \cite{pami/LiZ15} and Safe Regression model(SAFER) in proposed in \cite{aaai/LiZZ17}. Although these methods also study the security, our SEC is totally different from them from at least the following aspects. (1) These methods are developed in the semi-supervised learning paradigm, which focus on the exploring of unlabeled data, while SEC emphasizes on the usage of new coming features. They are raised from different disciplines with different goals. (2) Both S4VM and SAFER try to learn a well-performed semi-supervised base classifiers, while SEC aims at adapting classifiers on hand by utilizing data with new features. (3) Previous safe semi-supervised methods investigate only one-type of classifiers. For example, S4VM is the safe version of Semi-Supervised SVM and SAFER is the extension of traditional regression model. SEC can take all types of classifiers as the candidates. It is a general method, without specifying the classifier type.

As for the effectiveness of using multiple features, there are some investigations in multi-view learning. Commonly, there are two significant principles ensuring the success of multi-view learning, i.e., consensus and complementary principles \cite{xu2013MVLsurvey}. Consensus principle aims to maximize the agreement on multiple distinct views. For example, Dasgupta et al. have shown the connection between the consensus of two hypotheses on two views and their error rates \cite{DasguptaLM01}. Complementary principle requires that each view of the data may contain some knowledge that other views do not have. For example, Wang and Zhou have explained why co-training style algorithms succeed when there are no redundant views \cite{ecml/WangZ07}. Nevertheless, (1) different from these studies, we investigate the effectiveness of utilizing new features in another perspective. We aim to design a secure classifier while these approaches aim to find when the muti-view learning methods work. (2) Traditional multi-view learning approaches assume that all views are available while we assume that the data are coming in sequence.

As for diagnostic classification of schizophrenia, the machine learning algorithms have been successfully applied \cite{Arbabshirani2013Classification, Du2012High, Shen2009, Watanabe2014Disease}. For example, Du et al. have combined Kernel Principal Component Analysis (KPCA) and group Independent Component Analysis (ICA) to aid diagnosis of schizophrenia by using fMRI data acquired from an auditory oddball task paradigm \cite{Du2012High}. We have introduced an unsupervised learning-based classifier to discriminate patients by applying a combination of nonlinear dimensionality reduction and self-organized clustering algorithms \cite{Shen2009}. In addition, the altered resting-state functional network connectivity (FNC) among auditory, frontal-parietal, default-mode, visual, and motor networks have been adopted for classification of schizophrenia patients by using K-nearest neighbors classifier \cite{Arbabshirani2013Classification}. A recent schizophrenia classification challenge demonstrated clearly, across a broad range of classification approaches, the value of fMRI data in capturing useful information about this disease \cite{Silva2014}. There are at least two differences of SEC from these approaches. (1) They are designed with different motivations. SEC aims at learning a secure classifier while traditional methods are designed to achieve higher accuracies. (2) In most of traditional methods, only one type of features is employed. While in SEC, we consider the augmented features.

\section{Optimization}
\label{sec_optimization}

There are totally two optimization problems. i.e., Eq. (\ref{eq6}) and Eq. (\ref{eq10}). These problems are complicated. There are three groups of variables in Eq. (\ref{eq6}), i.e., $\mathbf{W}$ and $\mathbf{b}$, $\mathbf{M}$, $\lambda_1$ and $\lambda_2$ and the loss itself is sophisticated. The objective function in Eq. (\ref{eq10}) is also hard to solve due to the constraint.

\subsection{Solving Eq. (\ref{eq6})}

There are two parts in Eq. (\ref{eq6}). The difficulty lies in the capped hinge loss. Denote
\begin{equation}
\label{eq15}
\small{
\begin{split}
  &\xi_i(\mathbf{W},\mathbf{b}, \mathbf{m}_i, \lambda_1,\lambda_2) =  \\
  &\| \lambda_1 f_j(\mathbf{x}_i^{(1)})+\lambda_2 (\mathbf{W}^T \Phi(\mathbf{x}_i^{(2)})+\mathbf{b}) - \mathbf{y}_i-\mathbf{y}_i \circ \mathbf{m}_i\|_2^2 \\
  & h_i(\xi_i(\mathbf{W},\mathbf{b}, \mathbf{m}_i, \lambda_1,\lambda_2)) = \min ( \sqrt{\xi_i(\mathbf{W},\mathbf{b}, \mathbf{m}_i, \lambda_1,\lambda_2)}, 1).
\end{split}}
\end{equation}

Then the objective function in Eq. (\ref{eq6}) can be reformulated as $\lambda \|\mathbf{W}\|_{\rm{F}}^{2} + h_i(\xi_i(\mathbf{W},\mathbf{b}, \mathbf{m}_i, \lambda_1, \lambda_2))$. Note that $h_i$ is a concave function, we can borrow the optimization proposal in \cite{NieWH17} to solve this problem. Concretely, by calculating the supergradient of the concave function $h_i$ at point $\xi_i(\mathbf{W}, \mathbf{b}, \mathbf{m}_i, \lambda_1, \lambda_2)$, we replace the solution to the problem in Eq. (\ref{eq6}) by optimizing
\begin{equation}
\label{eq16}
\small{
\begin{split}
& \arg \min_{\mathbf{W},\mathbf{b}, \mathbf{M} \geq 0, \lambda_1,\lambda_2} \lambda \|\mathbf{W}\|_{\rm{F}}^{2} +  \\
& \sum_{i=1}^n \theta _i \| \lambda_1 f_j(\mathbf{x}_i^{(1)})+\lambda_2 (\mathbf{W}^T \Phi(\mathbf{x}_i^{(2)})+\mathbf{b}) - \mathbf{y}_i-\mathbf{y}_i \circ \mathbf{m}_i\|_2^2,
\end{split}}
\end{equation}
with
\begin{equation}
\label{eq17}
\theta_i = \left\{
{\begin{array}{*{20}{l}}
{\frac{1}{2} \xi_i(\mathbf{W},\mathbf{b}, \mathbf{m}_i, \lambda_1,\lambda_2)^{-\frac{1}{2}},~\xi_i^{\frac{1}{2}} \leq 1}\\
{0,~~~~~~~~~~~~~~~~~~~~~~~~~~~~~~~~~~~~~~\rm{otherwise}}.
\end{array}} \right.
\end{equation}

We apply the alternative strategy to optimize four groups of variables in Eq. (\ref{eq16}), $\mathbf{W}$ and $\mathbf{b}$, $\mathbf{m}_i$, $\lambda_1$ and $\lambda_2$, $\theta_i$, iteratively. Although the updating rule of $\theta_i$ is listed in Eq. (\ref{eq17}), it cannot be updated directly since we do not know the mapping function $\Phi$.

\subsubsection{Fixing $\mathbf{m}_i$, $\lambda_1$, $\lambda_2$, $\theta_i$ and optimizing $\mathbf{W}$, $\mathbf{b}$}

Denote $\mathbf{z}_i = \mathbf{y}_i+\mathbf{y}_i \circ \mathbf{m}_i-\lambda_1 f_j(\mathbf{x}_i^{(1)})$ and $\mathbf{Z} \in \mathbb{R}^{c\times n}$ with the $i$-th column as $\mathbf{z}_i$. Denote $\mathbf{e} \in \mathbb{R}^{n\times 1}$ the vector with all elements as 1. The matrix form of the optimization problem in Eq. (\ref{eq16}) is
\begin{equation}
\label{eq18}
\small{
\begin{split}
 \arg \min_{\mathbf{W},\mathbf{b}}~~\lambda \|\mathbf{W}\|_{\rm{F}}^{2} +  & \rm{Tr} \left( (\lambda_2 \mathbf{W}^T \Phi(\mathbf{X}_{tr}^{(2)})+ \lambda_2 \mathbf{b}\mathbf{e}^\top-\mathbf{Z})  \right.\\
  \Theta & \left. ( \lambda_2 \mathbf{W}^T \Phi(\mathbf{X}_{tr}^{(2)})+ \lambda_2 \mathbf{b}\mathbf{e}^\top-\mathbf{Z})^\top\right),
\end{split}}
\end{equation}
where $\Theta$ is a diagonal matrix with the $i$-th diagonal element as $\theta_i$ defined in Eq. (\ref{eq17}). $\Phi(\mathbf{X}_{tr}^{(2)}) \in \mathbb{R}^ {D\times n}$ is a matrix with the $i$-th column as $\Phi(\mathbf{x}_i^{(2)})$.

By taking the derivative of Eq. (\ref{eq18}) w.r.t. $\mathbf{b}$ and setting it to zeros, we have
\begin{equation}
\label{eq19}
\small{
\begin{split}
\mathbf{b} = \frac{1}{\lambda_2 \mathbf{e}^\top \Theta \mathbf{e}} \mathbf{Z} \Theta \mathbf{e} - \frac{1}{\mathbf{e}^\top \Theta \mathbf{e}} \mathbf{W}^T \Phi(\mathbf{X}_{tr}^{(2)}) \Theta \mathbf{e}.
\end{split}}
\end{equation}

Substituting Eq. (\ref{eq19}) into Eq. (\ref{eq18}) and setting the derivative w.r.t. $\mathbf{W}$ to zeros, we get
\begin{equation}
\label{eq20}
\small{
\begin{split}
\mathbf{W} = (\lambda_2^2 \Phi(\mathbf{X}_{tr}^{(2)}) \mathbf{H} (\Phi(\mathbf{X}_{tr}^{(2)}))^\top +\lambda\mathbf{I})^{-1} \lambda_2 \Phi(\mathbf{X}_{tr}^{(2)}) \mathbf{H}\mathbf{Z}^\top,
\end{split}}
\end{equation}
where $\mathbf{H}= \Theta -\frac{1}{\mathbf{e}^\top \Theta \mathbf{e}} \Theta \mathbf{e} \mathbf{e}^\top \Theta$ and $\mathbf{I}$ is an identity matrix.

The solutions in Eq. (\ref{eq19}) and Eq. (\ref{eq20}) cannot be calculated directly since we do not know the concrete form of $\Phi$. When testing data $\mathbf{X}_{te}^{(2)}$ comes, we only need to compute $\mathbf{W}^\top \Phi(\mathbf{X}_{te}^{(2)})+\mathbf{b}\tilde{\mathbf{e}}$ with
$\tilde{\mathbf{e}} \in \mathbb{R}^{t\times 1}$ the vector with all elements as 1. We now derive them by kernel trick. Notice that the matrix equality $(\mathbf{AB}+\mathbf{C})^{-1} \mathbf{A} = \mathbf{C}^{-1} \mathbf{A} (\mathbf{B}\mathbf{C}^{-1}\mathbf{A}+ \mathbf{I})^{-1}$. Denote $\mathbf{A} = \lambda_2^2 \Phi(\mathbf{X}_{tr}^{(2)}) \mathbf{H}$, $\mathbf{B} = \Phi(\mathbf{X}_{tr}^{(2)})$ and $\mathbf{C}=\lambda \mathbf{I}$, Eq. (\ref{eq20}) becomes
\begin{equation}
\label{eq21}
\small{
\begin{split}
& \mathbf{W} = \\
& (\lambda \mathbf{I})^{-1} \lambda_2 \Phi(\mathbf{X}_{tr}^{(2)}) \mathbf{H}
\left( (\Phi(\mathbf{X}_{tr}^{(2)}))^\top (\lambda \mathbf{I})^{-1} \lambda_2^2 \Phi(\mathbf{X}_{tr}^{(2)}) \mathbf{H} + \mathbf{I} \right)^{-1} \mathbf{Z}^\top\\
& = \lambda_2 \Phi(\mathbf{X}_{tr}^{(2)}) \mathbf{H} \left( \lambda_2^2 (\Phi(\mathbf{X}_{tr}^{(2)}))^\top \Phi(\mathbf{X}_{tr}^{(2)})\mathbf{H} +\lambda \mathbf{I} \right)^{-1}  \mathbf{Z}^\top\\
& = \lambda_2 \Phi(\mathbf{X}_{tr}^{(2)}) \mathbf{H} \left( \lambda_2^2 \mathbf{K}_{tr}^{(2)} \mathbf{H} +\lambda \mathbf{I} \right)^{-1}  \mathbf{Z}^\top.\\
\end{split}}
\end{equation}
Here $\mathbf{K}_{tr}^{(2)}$ is the kernel matrix determined by $\mathbf{X}_{tr}^{(2)}$. Now,
\begin{equation}
\label{eq22}
\small{
\begin{split}
& \mathbf{W}^\top \Phi(\mathbf{X}_{te}^{(2)}) =
\lambda_2 \mathbf{Z} \left( \lambda_2^2 \mathbf{K}_{tr}^{(2)} \mathbf{H} +\lambda \mathbf{I} \right)^{-1} \mathbf{H} \Phi(\mathbf{X}_{tr}^{(2)})^\top \Phi(\mathbf{X}_{te}^{(2)}).
\end{split}}
\end{equation}
Here $\mathbf{K}_{tr,te}^{(2)} \triangleq \Phi(\mathbf{X}_{tr}^{(2)})^\top \Phi(\mathbf{X}_{te}^{(2)})$ can be computed by employing the kernel trick.

Substituting Eq. (\ref{eq21}) into Eq. (\ref{eq19}), we get
\begin{equation}
\label{eq23}
\small{
\begin{split}
\mathbf{b} = \frac{\mathbf{ZDe}}{\lambda_2 \mathbf{e}^\top \Theta \mathbf{e}}-
\frac{\lambda_2}{\mathbf{e}^\top \Theta \mathbf{e}}
 \mathbf{Z} \left( \lambda_2^2 \mathbf{K}_{tr}^{(2)} \mathbf{H} +\lambda \mathbf{I} \right)^{-1} \mathbf{H} \mathbf{K}_{tr}^{(2)} \Theta \mathbf{e},
\end{split}}
\end{equation}
which can be derived by kernel trick directly.

\subsubsection{Fixing $\mathbf{W}$, $\mathbf{b}$, $\lambda_1$, $\lambda_2$, $\theta_i$ and optimizing $\mathbf{m}_i$}

When updating $\mathbf{m}_i$, we can optimize the problem in Eq. (\ref{eq16}) w.r.t. each $\mathbf{m}_i$ by solving the following problem.
\begin{equation}
\label{eq24}
\small{
\begin{split}
\arg \min_{\mathbf{m}_i \geq 0} \| \lambda_1 f_j(\mathbf{x}_i^{(1)})+\lambda_2 (\mathbf{W}^T \Phi(\mathbf{x}_i^{(2)})+\mathbf{b}) - \mathbf{y}_i-\mathbf{y}_i \circ \mathbf{m}_i\|_2^2.
\end{split}}
\end{equation}

Since $\mathbf{y}_i \in \{-1,1\}$, we multiply Eq. (\ref{eq24}) by $\mathbf{y}_i$ and reformulate it as
\begin{equation}
\label{eq25}
\small{
\begin{split}
\min_{\mathbf{m}_i \geq 0} \| \lambda_1 \mathbf{y}_i\circ f_j(\mathbf{x}_i^{(1)})+ \lambda_2 \mathbf{y}_i\circ (\mathbf{W}^T \Phi(\mathbf{x}_i^{(2)})+\mathbf{b}) - \bar{\mathbf{e}}- \mathbf{m}_i\|_2^2,
\end{split}}
\end{equation}
where $\bar{\mathbf{e}} \in \mathbb{R}^{c\times 1}$ is the vector with all elements as 1.

The optimal solution to Eq. (\ref{eq25}) is
\begin{equation}
\label{eq26}
\small{
\begin{split}
\mathbf{m}_i= \delta(\lambda_1 \mathbf{y}_i\circ f_j(\mathbf{x}_i^{(1)})+ \lambda_2 \mathbf{y}_i\circ (\mathbf{W}^T \Phi(\mathbf{x}_i^{(2)})+\mathbf{b}) - \bar{\mathbf{e}} ),
\end{split}}
\end{equation}
where the $p$-th element of $\delta(\mathbf{m})$ is $\max\{m(p),0\}$.

Similar to Eq. (\ref{eq22}), we can compute $\mathbf{W}^T \Phi(\mathbf{x}_i^{(2)})$ as
\begin{equation}
\label{eq27}
\small{
\begin{split}
& \mathbf{W}^\top \Phi(\mathbf{x}_i^{(2)}) =
\lambda_2 \mathbf{Z} \left( \lambda_2^2 \mathbf{K}_{tr}^{(2)} \mathbf{H} +\lambda \mathbf{I} \right)^{-1} \mathbf{H} \Phi(\mathbf{X}_{tr}^{(2)})^\top \Phi(\mathbf{x}_i^{(2)})
\end{split}}
\end{equation}
where $\Phi(\mathbf{X}_{tr}^{(2)})^\top \Phi(\mathbf{x}_i^{(2)})$ is the $i$-th column of $\mathbf{K}_{tr}^{(2)}$.

\subsubsection{Fixing $\mathbf{W}$, $\mathbf{b}$, $\mathbf{m}_i$, $\theta_i$ and optimizing $\lambda_1$, $\lambda_2$}

When updating $\lambda_1$, $\lambda_2$, we go back to the original problem in Eq. (\ref{eq16}). To make the notations simple, we denote
\begin{equation}
\label{eq28}
\small{
\begin{split}
& \mathbf{a}_i = f_j(\mathbf{x}_i^{(1)}) \in \mathbf{R}^{c\times1} \\
& \mathbf{b}_i = \mathbf{W}^T \Phi(\mathbf{x}_i^{(2)})+\mathbf{b} \in \mathbf{R}^{c\times1}\\
& \mathbf{c}_i =  \mathbf{y}_i+\mathbf{y}_i \circ \mathbf{m}_i \in \mathbf{R}^{c\times1}.
\end{split}}
\end{equation}

The optimization problem concerning updating $\lambda_1$ and $\lambda_2$ becomes
\begin{equation}
\label{eq29}
\small{
\begin{split}
& \arg \min_{\lambda_1,\lambda_2}
& \sum_{i=1}^n \theta _i \| \lambda_1 \mathbf{a}_i +\lambda_2 \mathbf{b}_i - \mathbf{c}_i\|_2^2.
\end{split}}
\end{equation}

Take the derivative of Eq. (\ref{eq29}) w.r.t. $\lambda_1$, $\lambda_2$ and set them to zeros, we can get the updating rule of $\lambda_1$ and $\lambda_2$ as
\begin{equation}
\label{eq30}
\small{
\begin{split}
& \lambda_1 = \frac{\sum_{i=1}^n \theta_i (\mathbf{a}_i^\top \mathbf{c}_i \mathbf{b}_i^\top \mathbf{b}_i-\mathbf{b}_i^\top \mathbf{c}_i \mathbf{b}_i^\top \mathbf{a}_i)}
{\sum_{i=1}^n \theta_i (\mathbf{a}_i^\top \mathbf{a}_i \mathbf{b}_i^\top \mathbf{b}_i-\mathbf{a}_i^\top \mathbf{b}_i \mathbf{a}_i^\top \mathbf{b}_i)} \\
& \lambda_2 = \frac{\sum_{i=1}^n \theta_i (\mathbf{b}_i^\top \mathbf{c}_i \mathbf{a}_i^\top \mathbf{a}_i-\mathbf{a}_i^\top \mathbf{c}_i \mathbf{a}_i^\top \mathbf{b}_i)}
{\sum_{i=1}^n \theta_i (\mathbf{a}_i^\top \mathbf{a}_i \mathbf{b}_i^\top \mathbf{b}_i-\mathbf{a}_i^\top \mathbf{b}_i \mathbf{a}_i^\top \mathbf{b}_i)}.
\end{split}}
\end{equation}

According to the definition of $\mathbf{a}_i$, $\mathbf{b}_i$ and $\mathbf{c}_i$ in Eq. (\ref{eq28}), we can also calculate all of them by kernel trick.

\subsubsection{Fixing $\mathbf{W}$, $\mathbf{b}$, $\mathbf{m}_i$, $\lambda_1$, $\lambda_2$ and updating $\theta_i$}

The updating rule of $\theta_i$ is given in Eq. (\ref{eq17}). We need to compute $\xi_i(\mathbf{W}, \mathbf{b}, \mathbf{m}_i, \lambda_1, \lambda_2)$, whose definition is shown in Eq. (\ref{eq15}). Considering the results in Eq. (\ref{eq27}), we can also derive $\xi_i$ by kernel trick. Thus, $\theta_i$ can be updated using Eq. (\ref{eq17}) directly.

In summary, the procedure of A-stage of SEC is listed in Algorithm \ref{alg1}. The following propositions can guarantee the convergence of our iteration.
\begin{algorithm}[!t]
\caption{SEC: A-stage}
\label{alg1}
\small{
\begin{algorithmic}
\STATE Initialize $\lambda_1$=$\lambda_2$ = 1/2, $\Theta = \mathbf{I}$, $\mathbf{M}=\mathbf{0}$. Calculate the kernel matrices $\mathbf{K}_{tr, te}^{(2)}$ and $\mathbf{K}_{tr}^{(2)}$. \\
\STATE  \textbf{Repeat} \\
\STATE 1: Calculate $\mathbf{Z}$ with the $i$-th column as $\mathbf{z}_i = \mathbf{y}_i+\mathbf{y}_i \circ \mathbf{m}_i-\lambda_1 f_j(\mathbf{x}_i^{(1)})$. Calculate $\mathbf{H} = \Theta -\frac{1}{\mathbf{e}^\top \Theta \mathbf{e}} \Theta \mathbf{e} \mathbf{e}^\top \Theta$.  \\
\STATE 2: Update $\mathbf{b}$ using Eq. (\ref{eq23}).
\STATE 3: Update $\mathbf{W}^\top \Phi(\mathbf{X}_{te}^{(2)})$ and $\mathbf{W}^\top \Phi(\mathbf{x}_i^{(2)})$ using Eq. (\ref{eq22}) and Eq. (\ref{eq27}) respectively.
\\
\STATE 4: Update $\mathbf{m}_i$ using Eq. (\ref{eq26}).
\\
\STATE 5: Update $\lambda_1$ and $\lambda_2$ using Eq. (\ref{eq30}).
\\
\STATE 6: Update $\theta_i$ using Eq. (\ref{eq17}).
\\
\STATE \textbf{Until converges}
\end{algorithmic}}
\end{algorithm}

\begin{theorem} \label{proposition3}
The Algorithm \ref{alg1}, which optimizes the problem in Eq. (\ref{eq16}) alternatively, will decrease the objective function values of the problem in Eq. (\ref{eq6}) in each iteration until it converges.
\end{theorem}

The proof is listed in the supplementary. Briefly, we avoid solving the complicated original problem in Eq. (\ref{eq6}) by substituting it with Eq. (\ref{eq16}), which can be easily addressed. More importantly, this substitution is reasonable since it can decrease the objective function values of original problem in Eq. (\ref{eq6}). We show some numerical results in the supplementary. Furthermore, the solution also has following proposition.
\begin{theorem} \label{proposition4}
The Algorithm \ref{alg1} will converge to a stationary point of the problem in Eq. (\ref{eq6}), which is usually a local minimum.
\end{theorem}

The proof is listed in the supplementary. Besides, empirical evidences in the appdendix show that Algorithm \ref{alg1} converges fast and it usually converges within 20 iterations.

\subsection{Solving Eq. (\ref{eq10})}

The optimization problem in Eq. (\ref{eq10}) is designed for security. Since there is a label matrix constraint on $\mathbf{Y}$, it is a combinational optimization problem in essence. In the following, we relax this hard constraint to a soft one. In other words, we assume that $\mathbf{Y}$ satisfies $0 \leq \mathbf{Y} \leq 1, \bar{\mathbf{e}}^\top \mathbf{Y} = \hat{\mathbf{e}}$. Here $\bar{\mathbf{e}} \in \mathbb{R}^{c\times 1}$ and $\hat{\mathbf{e}} \in \mathbb{R}^{t\times 1}$ are vectors with all elements as 1. The reasons for this relaxation are (1) It will facilitate solving the optimization problem in Eq. (\ref{eq10}). (2) This kind of approximation will not deteriorate the final performance heavily. We will harden the soft label after optimization. Since $\bar{\mathbf{Y}}^{(k)}$ and $\hat{\mathbf{Y}}^{(k)}$ are 0-1 label matrices, the hardness will not take great influence on the feasibility. According to Proposition \ref{proposition2}, it is usually a secure solution.

After relaxation, Eq. (\ref{eq10}) changes to
\begin{equation}
\label{eq31}
\small{
\begin{split}
          & \arg\max_{\mathbf{Y}, \epsilon}~\epsilon\\
\rm{s.t.~} &\|\mathbf{Y}- \bar{\mathbf{Y}}^{(k)}\|_{\rm{F}}^2 = \min_j \|\hat{\mathbf{Y}}^{(j)}- \bar{\mathbf{Y}}^{(k)}\|_{\rm{F}}^2-\epsilon, ~\forall~k=1,2,\cdots, m\\
          & \epsilon \geq 0 \\
          & 0 \leq \mathbf{Y} \leq 1,~~\bar{\mathbf{e}}^\top \mathbf{Y} = \hat{\mathbf{e}}.
\end{split}}
\end{equation}

This problem is matrix based. At first, we reformulate it as vector-based. Denote the vectorization of $\mathbf{Y}$ as $\mathbf{y} \in \mathbb{R}^{ct\times 1}$, which is formulated by connecting all the columns of $\mathbf{Y}$ one by one. Similarly, the vectorization of $\bar{\mathbf{Y}}^{(k)}$ can be derived by the same way and it is denoted as $\bar{\mathbf{y}}^{(k)}$.

When $k$ is fixed, we define $ q_k \triangleq \min_j \|\hat{\mathbf{Y}}^{(j)}- \bar{\mathbf{Y}}^{(k)}\|_{\rm{F}}^2$ since it is deterministic. Note that the optimization variables in Eq. (\ref{eq31}) are $\mathbf{Y}$ and $\epsilon$, we would like to join them into a unified vector
$\left[ {\begin{array}{*{10}{c}}
\epsilon \\
\mathbf{y}
\end{array}} \right]$.
The vector reformulation of Eq. (\ref{eq31}) is
\begin{equation}
\label{eq32}
\small{
\begin{split}
          & \arg \min_{\mathbf{y}, \epsilon}~
\left[\begin{array}{c c}
-1 & \mathbf{0}_{1\times ct}
\end{array}\right]
\left[\begin{array}{c}
\epsilon \\
\mathbf{y}
\end{array}\right] \\
\rm{s.t.~}
& \left[\begin{array}{cc}
\epsilon & \mathbf{y}^\top
\end{array}\right]
\left[\begin{array}{cc}
{0}&\mathbf{0}_{1\times ct}\\
\mathbf{0}_{ct\times 1}&\mathbf{I}_{ct\times ct}
\end{array}\right]
\left[\begin{array}{c}
\epsilon \\ \mathbf{y}
\end{array}\right]+\\
& \left[\begin{array}{cc}
1 & -2(\bar{\mathbf{y}}^{(k)})^\top
\end{array}\right]
\left[\begin{array}{c}
\epsilon \\ \mathbf{y}
\end{array}\right]
+ (\bar{\mathbf{y}}^{(k)})^\top \bar{\mathbf{y}}^{(k)}-q_k \le 0\\
& \mathbf{0}_{(1+ct)\times 1} \leq
\left[\begin{array}{c}
\epsilon \\
\mathbf{y}
\end{array}\right]
\leq
\left[\begin{array}{c}
+\infty \\
\mathbf{e}_{ct\times 1}
\end{array}\right] \\
&
\left[\begin{array}{c c}
\mathbf{0}_{t\times1} & \mathbf{Q}
\end{array}\right]
\left[\begin{array}{c}
\epsilon \\
\mathbf{y}
\end{array}\right] = \hat{\mathbf{e}}.
\end{split}}
\end{equation}
Here, the subscript indicates the matrix/vector size. $\mathbf{e}_{ct\times 1}$ is a $ct$-dimensional vector whose elements are all 1.

In Eq. (\ref{eq32}), since each column of $\mathbf{Y}$ corresponds to the soft label of a point, its sum should be 1. Thus, the definition of $\mathbf{Q} \in \mathbb{R}^{t\times ct}$ is
\begin{equation*}
\small{
\begin{split}
\mathbf{Q} = \left[ {\begin{array}{*{20}{c}}
\bar{\mathbf{e}}^\top &\mathbf{0}&\mathbf{0}& \cdots &\mathbf{0}\\
\mathbf{0}&\bar{\mathbf{e}}^\top&\mathbf{0}& \cdots &\mathbf{0}\\
 \vdots & \vdots & \vdots & \vdots & \vdots \\
\mathbf{0}&\mathbf{0}&\mathbf{0}& \cdots &\bar{\mathbf{e}}^\top
\end{array}} \right].
\end{split}}
\end{equation*}

Then, $\small{\bar{\mathbf{e}}^\top \mathbf{Y} = \hat{\mathbf{e}}~\Leftrightarrow~ \mathbf{Q} \mathbf{y} = \hat{\mathbf{e}}}$.

The problem in Eq. (\ref{eq32}) is a quadratically constrained linear program. It can be effectively solved by modern optimization tools, such as the CVX toolbox \cite{cvx, gb08}.

After deriving the optimal solution to Eq. (\ref{eq32}), we should reshape the elements of $\mathbf{y}$ to $\mathbf{Y}$. The class index is determined by the maximum values of each column of $\mathbf{Y}$.

In summary, in I-stage, we integrate the inherited classifiers from A-stage by solving a quadratically constrained linear program in Eq. (\ref{eq32}). The
procedure of I-stage of SEC is listed in Algorithm \ref{alg2}.

\begin{algorithm}[!t]
\caption{SEC: I-stage}
\label{alg2}
\small{
\begin{algorithmic}
\STATE 1: Calculate $q_k=\min_j \|\hat{\mathbf{Y}}^{(j)}- \bar{\mathbf{Y}}^{(k)}\|_{\rm{F}}^2$. Reshape $\bar{\mathbf{Y}}^{(k)}$ to $\bar{\mathbf{y}}^{(k)}$ for $k=1,2,\cdots, m$.  \\
\STATE 2: Solve the optimization problem in Eq. (\ref{eq32}) and derive the optimal $\mathbf{y}$.
\STATE 3: Reshape $\mathbf{y}$ to $\mathbf{Y}$ and determine the labels of testing points.
\end{algorithmic}}
\end{algorithm}

\subsection{Computational Complexity Analysis}

As seen from the procedure in Algorithm \ref{alg1}, we optimize four group of variables alternatively. The most time consuming steps are the computation of $\mathbf{W}^\top \Phi(\mathbf{X}_{te}^{(2)})$, $\mathbf{W}^\top \Phi(\mathbf{x}_i^{(2)})$ and $\mathbf{b}$ shown in step 2 and step 3 of Algorithm \ref{alg1}. Since we need to calculate the inverse of an matrix with size $n$, the computational complexity is $O(n^3)$.

As for Algorithm \ref{alg2}, we need to solve a convex quadric constrained linear programming with $(t\times c+1)$ variables. There are a lot of modern technologies for solution, such as \cite{cvx}. We use a naive implementation by interior-point methods and the computational complexity of this problem is
$O((t\times c+1)^3)$ \cite{Boyd:2004}.

In summary, the total computational complexity of SEC is $O((\max(n,t\times c+1))^3)$. We will show some experimental results. Due to the limitation of space, the results are listed in the supplementary. Besides, as seen from the procedures of SEC in Algorithm 1 and Algorithm 2, SEC can be regarded as an online approach. In other words, if we have more than one type of new coming features, we only need to conduct the adaption and integration stages on the latest type of feature, without retraining it using all previous data. Compared with the batch methods, it saves time with the increase of new type features.

\section{Experiments}
\label{sec_exp}

\begin{table*}[!t]
    \centering
    \caption{Details of the data sets with different types of features used in our experiments (feature type(dimensionality)).}
    \label{table_DataDetail}
    \scriptsize{
    \begin{tabular}{c||c c c c c c c c c}
    \hline
    Feature type & SensIT       & NBA-NASCAR  & Trecvid2003  & Ionoshpere & Protein    & Digit   & AD     & MSRC-v1        & Caltech-7 \\
    \hline
    1            & Acoustic(50) & Image(1024) & HSV(165)     & (34)       & FFT(4910)  & FOU(76) & ALT(111)  & CMT(48)        & LBP(256)  \\
    2            & Seismic(50)  & Text(296)   & Text(1894)   & (25)       & GE(441)    & KAR(64) & AURL(472)  & CENTRIST(1302) & CENTRIST(1302) \\
    3            & -            & -           & -            &  -         & Pfam(3753) &FAC(216) & URL(457)  & GIST(512)      & GIST(512)  \\
    4            & -            & -           &-             & -          &  -      &-        &-       & HOG(100 )      & HOG(100)  \\
    \hline
    Data points  & 20000        & 840         & 1078         & 351       &  629     & 2000    & 3264   & 210            &  441 \\
    \hline
    Classes      & 2            & 2           & 5            & 2         &  2       & 10      &    2   & 7              &  7  \\
    \hline
\end{tabular}}
\vskip -0.0in
\end{table*}

\begin{table*}[!t]
\caption{Testing accuracies (mean$\pm$std) of the compared methods on 16 data sets with different percent of training and testing examples. '$\bullet/${\tiny $\odot$}$/\circ$' denote respectively that the adaption methods or SEC are significantly better/tied/worse than the best results conducted on only on hand features by the $t$-test\cite{t-test} with confidence level 0.05. The top half table presents results with 30\% training versus 70\% testing and the bottom half table shows results with 50\% training versus 50\% testing.}
\label{table_accuracy}
\centering
\vskip -0.1in
{
\begin{tabular}{c|| c c c c c c c c }
\hline
  Methods       &  Best   &  AdRegression   & AdKNN    & AdNaiveBayes    & AdBoosting &  AdSVM(Lin) & AdSVM(RBF) & SEC \\
\hline
SensIT          & .5017(.0045)& .7915(.0114)$\bullet$ & .7944(.0024)$\bullet$ & .8052(.0050)$\bullet$ & .7909(.0103)$\bullet$ & .8082(.0062)$\bullet$ & .7447(.1380)$\bullet$ & .8132(.0327)$\bullet$ \\
NBA-NASCAR      & .7405(.0158)& .9913(.0038)$\bullet$ & .9947(.0032)$\bullet$ & .9874(.0056)$\bullet$ & .9970(.0031)$\bullet$ & .9942(.0056)$\bullet$ & .9963(.0037)$\bullet$ & .9908(.0055)$\bullet$ \\
Trecvid2003     & .6183(.0199)& .8347(.0172)$\bullet$ & .8318(.0179)$\bullet$ & .7763(.0250)$\bullet$ & .8184(.0279)$\bullet$ & .7948(.0271)$\bullet$ & .7989(.0253)$\bullet$ & .8218(.0216)$\bullet$ \\
Ionoshpere      & .9094(.0176)& .9864(.0169)$\bullet$ & .9872(.0100)$\bullet$ & .9668(.0389)$\bullet$ & .9592(.0345)$\bullet$ & .9829(.0192)$\bullet$ & .8721(.0761){\tiny $\odot$} & .9418(.0465)$\bullet$ \\
Protein1add2    & .8428(.0343)& .9548(.0295)$\bullet$ & .9861(.0114)$\bullet$ & .9833(.0132)$\bullet$ & .9783(.0201)$\bullet$ & .9474(.0458)$\bullet$ & .9560(.0504)$\bullet$ & .9363(.0516)$\bullet$ \\
Protein1add3    & .8426(.0312)& .8901(.0234)$\bullet$ & .8841(.0287)$\bullet$ & .8904(.0165)$\bullet$ & .9167(.0260)$\bullet$ & .8836(.0232)$\bullet$ & .9070(.0572)$\bullet$ & .8842(.0247)$\bullet$ \\
Digit1add2      & .8021(.0149)& .7933(.0376){\tiny $\odot$} & .8127(.0378){\tiny $\odot$} & .7443(.0173)$\circ$ & .6550(.0913)$\circ$ & .8586(.0103)$\bullet$ & .8580(.0121)$\bullet$ & .8617(.0160)$\bullet$ \\
Digit1add3      & .8100(.0093)& .8794(.0472)$\bullet$ & .9286(.0137)$\bullet$ & .7563(.0064)$\circ$ & .8491(.0057)$\bullet$ & .9420(.0130)$\bullet$ & .9419(.0139)$\bullet$ & .9447(.0146)$\bullet$ \\
AD1add2         & .8993(.0065)& .9569(.0044)$\bullet$ & .9567(.0039)$\bullet$ & .9569(.0048)$\bullet$ & .9577(.0049)$\bullet$ & .9551(.0053)$\bullet$ & .9528(.0050)$\bullet$ & .9565(.0047)$\bullet$ \\
AD1add3         & .8994(.0047)& .9421(.0035)$\bullet$ & .9424(.0035)$\bullet$ & .9406(.0034)$\bullet$ & .9410(.0034)$\bullet$ & .9424(.0029)$\bullet$ & .9410(.0036)$\bullet$ & .9398(.0033)$\bullet$ \\
MSRC-v1-1add2   & .6551(.0573)& .8133(.0344)$\bullet$ & .8463(.0315)$\bullet$ & .5160(.1056)$\circ$ & .8245(.0349)$\bullet$ & .7531(.0520)$\bullet$ & .7677(.0424)$\bullet$ & .8510(.0311)$\bullet$ \\
MSRC-v1-1add3   & .6612(.0778)& .8153(.0619)$\bullet$ & .8388(.0551)$\bullet$ & .4942(.1018)$\circ$ & .8241(.0608)$\bullet$ & .7738(.0654)$\bullet$ & .7939(.0616)$\bullet$ & .8408(.0577)$\bullet$ \\
MSRC-v1-1add4   & .6755(.0656)& .7041(.0515){\tiny $\odot$} & .7701(.0669)$\bullet$ & .4136(.0657)$\circ$ & .7527(.0612)$\bullet$ & .6789(.0697){\tiny $\odot$} & .6861(.0489){\tiny $\odot$} & .7772(.0661)$\bullet$ \\
Caltech-7-1add2 & .7000(.0484)& .8916(.0211)$\bullet$ & .8794(.0229)$\bullet$ & .7716(.0314)$\bullet$ & .8932(.0218)$\bullet$ & .8774(.0183)$\bullet$ & .8856(.0231)$\bullet$ & .8857(.0251)$\bullet$ \\
Caltech-7-1add3 & .6907(.0447)& .8550(.0224)$\bullet$ & .8653(.0268)$\bullet$ & .6927(.0668){\tiny $\odot$} & .8498(.0268)$\bullet$ & .8300(.0268)$\bullet$ & .8422(.0275)$\bullet$ & .8638(.0238)$\bullet$ \\
Caltech-7-1add4 & .7002(.0394)& .8128(.0269)$\bullet$ & .8333(.0317)$\bullet$ & .6885(.0390){\tiny $\odot$} & .8167(.0310)$\bullet$ & .7883(.0383)$\bullet$ & .7919(.0414)$\bullet$ & .8244(.0270)$\bullet$ \\
\hline \hline
SensIT          & .6013(.0057)& .7957(.0067)$\bullet$ & .7950(.0021)$\bullet$ & .8039(.0030)$\bullet$ & .7922(.0061)$\bullet$ & .8032(.0056)$\bullet$ & .8012(.0050)$\bullet$ & .8044(.0023)$\bullet$ \\
NBA-NASCAR      & .7506(.0151)& .9976(.0022)$\bullet$ & .9990(.0020)$\bullet$ & .9945(.0040)$\bullet$ & .9994(.0013)$\bullet$ & .9990(.0016)$\bullet$ & .9994(.0011)$\bullet$ & .9986(.0020)$\bullet$ \\
Trecvid2003     & .6313(.0195)& .8817(.0178)$\bullet$ & .8750(.0145)$\bullet$ & .8248(.0226)$\bullet$ & .8746(.0208)$\bullet$ & .8578(.0226)$\bullet$ & .8548(.0214)$\bullet$ & .8803(.0183)$\bullet$ \\
Ionoshpere      & .9157(.0177)& .9920(.0180)$\bullet$ & .9957(.0113)$\bullet$ & .9934(.0167)$\bullet$ & .9343(.0448){\tiny $\odot$} & .9946(.0157)$\bullet$ & .9806(.0364)$\bullet$ & .9971(.0047)$\bullet$ \\
Protein1add2    & .8529(.0162)& .9782(.0116)$\bullet$ & .9941(.0038)$\bullet$ & .9911(.0094)$\bullet$ & .9925(.0035)$\bullet$ & .9545(.0634)$\bullet$ & .8990(.0629)$\bullet$ & .9817(.0151)$\bullet$ \\
Protein1add3    & .8548(.0227)& .9212(.0146)$\bullet$ & .9363(.0237)$\bullet$ & .9277(.0179)$\bullet$ & .9462(.0149)$\bullet$ & .9124(.0253)$\bullet$ & .9041(.0484)$\bullet$ & .9159(.0303)$\bullet$ \\
Digit1add2      & .8138(.0084)& .8622(.0383)$\bullet$ & .8468(.0115)$\bullet$ & .7524(.0098)$\circ$   & .7664(.0746){\tiny $\odot$} & .8698(.0216)$\bullet$ & .8622(.0288)$\bullet$ & .8850(.0272)$\bullet$ \\
Digit1add3      & .8218(.0148)& .9482(.0155)$\bullet$ & .9378(.0105)$\bullet$ & .7582(.0142)$\circ$ & .8588(.0061)$\bullet$ & .9462(.0126)$\bullet$ & .9410(.0122)$\bullet$ & .9550(.0076)$\bullet$ \\
AD1add2         & .9012(.0025)& .9567(.0039)$\bullet$ & .9569(.0042)$\bullet$ & .9582(.0041)$\bullet$ & .9570(.0044)$\bullet$ & .9592(.0048)$\bullet$ & .9577(.0046)$\bullet$ & .9569(.0047)$\bullet$ \\
AD1add3         & .9039(.0041)& .9454(.0042)$\bullet$ & .9457(.0045)$\bullet$ & .9445(.0048)$\bullet$ & .9443(.0040)$\bullet$ & .9461(.0054)$\bullet$ & .9460(.0058)$\bullet$ & .9444(.0043)$\bullet$ \\
MSRC-v1-1add2   & .7605(.0749)& .8681(.0353)$\bullet$ & .8967(.0390)$\bullet$ & .6824(.0914)$\circ$ & .8733(.0362)$\bullet$ & .8200(.0434)$\bullet$ & .8276(.0345)$\bullet$ & .8957(.0408)$\bullet$ \\
MSRC-v1-1add3   & .7386(.0530)& .8938(.0309)$\bullet$ & .8919(.0302)$\bullet$ & .7105(.0820){\tiny $\odot$}& .8967(.0374)$\bullet$ & .8352(.0371)$\bullet$ & .8490(.0388)$\bullet$ & .8952(.0295)$\bullet$ \\
MSRC-v1-1add4   & .7390(.0274)& .7257(.0297){\tiny $\odot$} & .8171(.0483)$\bullet$ & .4514(.0570)$\circ$ & .8076(.0297)$\bullet$ & .7238(.0530){\tiny $\odot$} & .6990(.0306){\tiny $\odot$} & .8324(.0341)$\bullet$ \\
Caltech-7-1add2 & .7482(.0233)& .9091(.0243)$\bullet$ & .9127(.0099)$\bullet$ & .8355(.0171)$\bullet$ & .9091(.0144)$\bullet$ & .8827(.0189)$\bullet$ & .8991(.0174)$\bullet$ & .8836(.0460)$\bullet$ \\
Caltech-7-1add3 & .7727(.0252)& .8666(.0295)$\bullet$ & .8882(.0283)$\bullet$ & .7659(.0849){\tiny $\odot$} & .8620(.0212)$\bullet$ & .8648(.0359)$\bullet$ & .8659(.0371)$\bullet$ & .8841(.0332)$\bullet$ \\
Caltech-7-1add4 & .7441(.0455)& .8225(.0331)$\bullet$ & .8482(.0343)$\bullet$ & .6948(.0629)$\circ$ & .8214(.0347)$\bullet$ & .8043(.0353)$\bullet$ & .7911(.0405)$\bullet$ & .8293(.0406)$\bullet$ \\
\hline
\multicolumn{2}{c}{Win/Tie/Loss}& 29/3/0 & 31/1/0 & 18/4/10 & 29/2/1 & 30/2/0 & 29/3/0 & 32/0/0  \\
\hline
\end{tabular}}
\vskip -0.1in
\end{table*}

In this section, we perform experiments to evaluate the performance and efficiency of SEC. There are four groups of experiments. In the first group, to validate the security of SEC, we compare it with the best single-view classification result and the adaption results. In the second group, we compare SEC with several popular multi-view classification approaches. After that, to show the effectiveness of adaption, we list the classification accuracies before and after adaption of different types of classifiers. The numerical comparison is presented to verify the condition in Proposition \ref{proposition2}. Finally, we apply our approach in the application of diagnostic classification of schizophrenia. Before going into the details, let us introduce the data sets at first.

\subsection{Configuration}

There are total 16 classification results conducted on 9 different data sets. The brief summary of these data sets are listed in Table \ref{table_DataDetail} and the detailed description of each data sets are listed as follows.

\textbf{SensIT}\footnote{ http://www.ecs.umass.edu/~mduarte/Software.html} contains data from wireless distributed sensor networks. It is collected from two different types of sensors, that is, acoustic and seismic sensor to record different signals for three types of vehicle (three classes) in an intelligent transportation system. The first view data has 50-dimensional acoustic features and the new coming features are collected from the seismic sensors. For demonstration, we select the first 10000 data points from the first two classes and result in 20000 data point in 2 classes.

\textbf{NBA-NASCAR}\footnote{www.cst.ecnu.edu.cn/~slsun/software/MvLapSVMcode.zip} is collected from the sports gallery of the Yahoo website in 2008. Following \cite{apin/XieS14}, this data set consists of 420 NBA images and 420 NASCAR images. For each image, there is an attached short text describing information.  The first view is the gray features of each image, which is normalized to have 1024 gray features. The new coming features are obtained from attached short text and 296-dimensional TFIDF features have been extracted.

\textbf{TRECVID2003}\footnote{http://bigml.cs.tsinghua.edu.cn/~ningchen/data.htm} is a video data set, provided by the author of \cite{ChenZSX12}, which is composed of 1078 manually labeled video shots which belong to 5 categories. Each shot has two different representations, which are extracted from the associated key frame. The first is 165-dimensional vector of HSV color histogram. The new coming feature is the 1894-dimensional vector of text.

\textbf{Ionoshpere}\footnote{http://archive.ics.uci.edu/ml/datasets/Ionosphere} is collected by a system in Goose Bay, Labrador. This system
consists of a phased array of 16 high-frequency antennas and results in 34-dimensional observations. It includes 351 instances in total which are divided into 225 'Good' (positive) instances and 126 'Bad (negative) instances. As for new coming features, we capture all the data variance while reducing the dimensionality from 34 to 25 with PCA.

\textbf{Protein}\footnote{https://noble.gs.washington.edu/proj/sdp-svm/} consists of 629 yeast proteins and is divided into 2 classes: 497 membrane proteins and 132 ribosomal proteins. Each protein is represented by a 4910-dimensional vector of FFT features. There are two types of new coming features. The first is a 441 dimensional vector of gene expression (GE) features and the second is a 3735-dimensional vector of the Pfam features.

\textbf{Digit}\footnote{https://archive.ics.uci.edu/ml/datasets/Multiple+Features} contains 2,000 data points for 0 to 9 ten digit classes and each class has 200 data points. The 76 Fourier coefficients of the character shapes (FOU) are the on hand descriptions. We use another two types of descriptions as the new coming features. The first is 64-dimensional Karhunen-love coefficients (KAR) and the second is the 216-dimensional profile correlations (FAC).

\textbf{AD}\footnote{http://archive.ics.uci.edu/ml/datasets/Internet+Advertisements} is a set of possible advertisements on web pages. The task is to predict whether a web is an advertisement or not. This data set contains 3264 examples, among which 458 examples are advertisements. The first is the 111-dimensional descriptions concerning information of the alt terms (ALT). The first new coming features are the 472-dimensional descriptions of anchor text in ancurl (AURL). The second new coming features are the 457-dimensional descriptions of the phrases occurring in the URL

\textbf{MSRC-v1}\footnote{https://www.microsoft.com/en-us/research/project/} data set consists of 240 images and is divided into 8 classes. Following \cite{ijcv/LeeG09}, we select 7 classes composed of tree, building, airplane, cow, face, car, bicycle, and each class has 30 images. We extract different kinds of descriptions. The on hand descriptions are the 48-dimensional color moment (CMT) features. There are three types of new coming features. The first one is 1302-dimensional CENTRIST feature. The second is 512-dimensional GIST feature and the last is 100-dimensional HOG feature.

\begin{table*}[!t]
\caption{F-score results (mean$\pm$std) of the compared methods on 16 data sets with different number of training and testing examples. The other settings are the same as that in Table \ref{table_accuracy}.}
\label{table_fscore}
\centering
\vskip -0.1in
{
\begin{tabular}{c|| c c c c c c c c }
\hline
  Methods       &  Best   &  AdRegression   & AdKNN    & AdNaiveBayes    & AdBoosting &  AdSVM(Lin) & AdSVM(RBF) & SEC \\
\hline
SensIT          & .4959(.0099) & .7903(.0118)$\bullet$ & .7940(.0024)$\bullet$ & .8044(.0052)$\bullet$ & .7899(.0106)$\bullet$ & .8077(.0059)$\bullet$ & .7108(.0116)$\bullet$ & .8072(.0081)$\bullet$ \\
NBA-NASCAR      & .7397(.0158) & .9913(.0038)$\bullet$ & .9947(.0032)$\bullet$ & .9874(.0056)$\bullet$ & .9970(.0031)$\bullet$ & .9942(.0056)$\bullet$ & .9963(.0037)$\bullet$ & .9908(.0055)$\bullet$ \\
Trecvid2003     & .5139(.0308) & .8141(.0180)$\bullet$ & .8095(.0171)$\bullet$ & .7444(.0307)$\bullet$ & .7938(.0303)$\bullet$ & .7682(.0289)$\bullet$ & .7730(.0293)$\bullet$ & .7964(.0233)$\bullet$ \\
Ionoshpere      & .8971(.0198) & .9852(.0182)$\bullet$ & .9857(.0112)$\bullet$ & .9643(.0410)$\bullet$ & .9536(.0399)$\bullet$ & .9808(.0213)$\bullet$ & .8622(.0948){\tiny $\odot$} & .9358(.0501)$\bullet$ \\
Protein1add2    & .7794(.0293) & .9225(.0591)$\bullet$ & .9781(.0195)$\bullet$ & .9740(.0211)$\bullet$ & .9648(.0351)$\bullet$ & .9022(.0991)$\bullet$ & .9120(.1245)$\bullet$ & .9453(.0256)$\bullet$ \\
Protein1add3    & .7789(.0227) & .8036(.0540){\tiny $\odot$} & .7865(.0627){\tiny $\odot$} & .8297(.0262)$\bullet$ & .8568(.0524)$\bullet$ & .8042(.0642){\tiny $\odot$} & .8452(.0983)$\bullet$ & .8403(.0539)$\bullet$ \\
Digit1add2      & .7899(.0191) & .7509(.0477){\tiny $\odot$} & .8119(.0381){\tiny $\odot$} & .7430(.0143)$\circ$ & .5773(.1142)$\circ$ & .8572(.0113)$\bullet$  & .8572(.0127)$\bullet$  & .8604(.0169)$\bullet$  \\
Digit1add3      & .8040(.0167) & .8569(.0677){\tiny $\odot$} & .9294(.0138)$\bullet$ & .7555(.0045)$\circ$ & .8101(.0052){\tiny $\odot$} & .9428(.0122)$\bullet$ & .9427(.0130)$\bullet$ & .9452(.0141)$\bullet$ \\
AD1add2         & .7131(.0206) & .8974(.0104)$\bullet$ & .8970(.0089)$\bullet$ & .8971(.0111)$\bullet$ & .8994(.0119)$\bullet$ & .8968(.0107)$\bullet$ & .8913(.0112)$\bullet$ & .8964(.0115)$\bullet$\\
AD1add3         & .7149(.0202) & .8564(.0088)$\bullet$ & .8578(.0091)$\bullet$ & .8510(.0094)$\bullet$ & .8520(.0094)$\bullet$ & .8610(.0083)$\bullet$ & .8569(.0128)$\bullet$ & .8480(.0116)$\bullet$ \\
MSRC-v1-1add2   & .6419(.0645) & .8141(.0337)$\bullet$ & .8464(.0316)$\bullet$ & .4992(.1093)$\circ$ & .8253(.0337)$\bullet$ & .7518(.0514)$\bullet$ & .7679(.0417)$\bullet$ & .8517(.0299)$\bullet$ \\
MSRC-v1-1add3   & .6516(.0631) & .8104(.0639)$\bullet$ & .8350(.0541)$\bullet$ & .4764(.1165)$\circ$ & .8199(.0622)$\bullet$ & .7682(.0680)$\bullet$ & .7896(.0634)$\bullet$ & .8375(.0571)$\bullet$ \\
MSRC-v1-1add4   & .6632(.0571) & .6773(.0657){\tiny $\odot$} & .7538(.0779)$\bullet$ & .3730(.0809)$\circ$ & .7328(.0721)$\bullet$ & .6498(.0848){\tiny $\odot$} & .6572(.0608){\tiny $\odot$} & .7616(.0763)$\bullet$ \\
Caltech-7-1add2 &.6094(.0574) & .8389(.0342)$\bullet$ & .8201(.0338)$\bullet$ & .6793(.0428)$\bullet$ & .8374(.0334)$\bullet$ & .8242(.0239)$\bullet$ & .8350(.0325)$\bullet$ & .8274(.0375)$\bullet$ \\
Caltech-7-1add3 & .5810(.0746) & .7673(.0435)$\bullet$  & .7966(.0466)$\bullet$  & .5918(.0764)$\circ$ & .7518(.0478)$\bullet$  & .7826(.0331)$\bullet$  & .7940(.0354)$\bullet$  & .7849(.0422)$\bullet$  \\
Caltech-7-1add4 & .6084(.0745) & .7373(.0493)$\bullet$  & .7585(.0545)$\bullet$  & .5897(.0485){\tiny $\odot$} & .7400(.0523)$\bullet$ & .7166(.0583)$\bullet$ & .7195(.0619)$\bullet$ & .7480(.0477)$\bullet$ \\
\hline
\hline
SensIT          & .4896(.0059) & .7950(.0065)$\bullet$  & .7940(.0029)$\bullet$  & .8017(.0038)$\bullet$  & .7917(.0057)$\bullet$  & .8024(.0057)$\bullet$  & .8004(.0049)$\bullet$  & .8029(.0028)$\bullet$  \\
NBA-NASCAR      & .7499(.0151) & .9976(.0022)$\bullet$  & .9990(.0020)$\bullet$  & .9945(.0040)$\bullet$  & .9994(.0013)$\bullet$  & .9990(.0016)$\bullet$ & .9994(.0011)$\bullet$  & .9986(.0020)$\bullet$ \\
Trecvid2003     & .5554(.0289) & .8699(.0200)$\bullet$  & .8612(.0165)$\bullet$  & .8039(.0249)$\bullet$  & .8621(.0217)$\bullet$  & .8400(.0252)$\bullet$  & .8369(.0251)$\bullet$  & .8673(.0202)$\bullet$ \\
Ionoshpere      & .9113(.0187) & .9913(.0195)$\bullet$ & .9953(.0124)$\bullet$ & .9930(.0177)$\bullet$ & .9253(.0510){\tiny $\odot$} & .9939(.0176)$\bullet$ & .9779(.0435)$\bullet$ & .9969(.0050)$\bullet$\\
Protein1add2    & .7526(.0513) & .9659(.0198)$\bullet$ & .9911(.0058)$\bullet$ & .9861(.0157)$\bullet$ & .9887(.0052)$\bullet$ & .9283(.0943)$\bullet$ & .7778(.0700){\tiny $\odot$} & .9713(.0247)$\bullet$\\
Protein1add3    & .7682(.0434) & .8732(.0267)$\bullet$ & .8978(.0399)$\bullet$ & .8950(.0253)$\bullet$ & .9140(.0238)$\bullet$ & .8600(.0410)$\bullet$ & .8260(.0971)$\bullet$ & .8588(.0583)$\bullet$ \\
Digit1add2      & .8052(.0134) & .8408(.0563){\tiny $\odot$} & .8463(.0121)$\bullet$ & .7524(.0096)$\circ$ & .7108(.0623)$\circ$ & .8687(.0218)$\bullet$ & .8599(.0296)$\bullet$ & .8837(.0271)$\bullet$ \\
Digit1add3      & .8184(.0181) & .9487(.0147)$\bullet$ & .9389(.0098)$\bullet$ & .7620(.0112)$\circ$ & .8189(.0060){\tiny $\odot$} & .9467(.0121)$\bullet$ & .9416(.0120)$\bullet$ & .9553(.0071)$\bullet$\\
AD1add2         & .7374(.0130) & .9004(.0129)$\bullet$ & .9008(.0131)$\bullet$ & .9035(.0128)$\bullet$ & .9007(.0132)$\bullet$ & .9088(.0131)$\bullet$ & .9044(.0134)$\bullet$ & .9005(.0141)$\bullet$ \\
AD1add3         & .7394(.0134) & .8689(.0111)$\bullet$ & .8689(.0118)$\bullet$ & .8651(.0141)$\bullet$ & .8654(.0114)$\bullet$ & .8723(.0127)$\bullet$ & .8720(.0135)$\bullet$ & .8647(.0129)$\bullet$ \\
MSRC-v1-1add2   & .7485(.0756) & .8674(.0354)$\bullet$ & .8962(.0397)$\bullet$ & .6827(.0992)$\circ$ & .8715(.0366)$\bullet$ & .8191(.0431)$\bullet$ & .8276(.0338)$\bullet$ & .8947(.0424)$\bullet$ \\
MSRC-v1-1add3   & .7312(.0571) & .8908(.0323)$\bullet$ & .8903(.0300)$\bullet$ & .7109(.0914)$\circ$ & .8942(.0388)$\bullet$ & .8331(.0366)$\bullet$ & .8466(.0395)$\bullet$ & .8929(.0320)$\bullet$ \\
MSRC-v1-1add4   & .7354(.0329) & .7101(.0495){\tiny $\odot$} & .8218(.0480)$\bullet$ & .4099(.0694)$\circ$ & .8075(.0301)$\bullet$ & .7089(.0584){\tiny $\odot$} & .6811(.0335)$\circ$ & .8361(.0356)$\bullet$ \\
Caltech-7-1add2 & .6403(.0526) & .8612(.0390)$\bullet$  & .8660(.0226)$\bullet$  & .7617(.0177)$\bullet$  & .8560(.0320)$\bullet$  & .8323(.0254)$\bullet$  & .8494(.0318)$\bullet$  & .8299(.0641)$\bullet$  \\
Caltech-7-1add3 & .6985(.0316) & .7990(.0564)$\bullet$ & .8338(.0424)$\bullet$ & .6823(.0563){\tiny $\odot$} & .7729(.0430)$\bullet$ & .8227(.0416)$\bullet$ & .8276(.0426)$\bullet$ & .8297(.0538)$\bullet$\\
Caltech-7-1add4 & .6609(.0469) & .7585(.0482)$\bullet$ & .7853(.0482)$\bullet$ & .6123(.0643)$\circ$ & .7568(.0524)$\bullet$ & .7407(.0542)$\bullet$ & .7285(.0585)$\bullet$ & .7679(.0570)$\bullet$ \\
\hline
\multicolumn{2}{c}{Win/Tie/Loss}& 26/6/0 & 30/2/0 & 18/2/12 & 27/3/2 & 29/3/0 & 28/3/1 & 32/0/0  \\
\hline
\end{tabular}}
\vskip -0.1in
\end{table*}

\textbf{Caltech-7} is a subsection of Caltech101\footnote{http://www.vision.caltech.edu/Image Datasets/Caltech101/}. Following \cite{iccv/DueckF07}, 7 widely used classes with 441 images were selected from the Caltech101. We extract the 256-dimensional LBP features on hand. Similar to the setting of MSRC-v1, three different descriptors have been extracted as the new coming features. The first one is 1302-dimensional CENTRIST feature. The second is 512-dimensional GIST feature and the last is 100-dimensional HOG feature.

We select six different types of popular classifiers as the baseline. They are Least Square Regression model, KNN classifier, Naive Bayes, Boosting, Linear SVM and RBF kernel SVM. Consequently, their classification results are denoted as $\hat{\mathbf{Y}}^{(k)}$ for $k=1,2,\cdots,6$. After adapting them by optimizing Eq. (\ref{eq6}), these methods are named as AdRegression, AdKNN, AdNaiveBayes, AdBoosting, AdSVM (Lin) and AdSVM (RBF) respectively. To evade the problem of parameter determination, the mapping function in Eq. (\ref{eq6}) is set as $\mathbf{\Phi(\mathbf{x})}=\mathbf{x}$.

As for parameter determination, there is only one additional parameter, i.e., $\lambda$ in Eq. (\ref{eq6}). We determine it by 5-fold cross validation for each baseline method. The baseline classifiers are provided by Matlab with the default parameters.

\subsection{Classification Performance Comparison}

To show the security of SEC, we compare it with the best classification results of 6 six baseline classifiers, which is denoted as Best. Besides, we also report the adaption version of all baseline methods. We first split all the data with different percentages of training and testing samples. The two ratios between training/testing are 0.3/0.7 and 0.5/0.5 respectively. With 50 independent random splitting, the mean classification accuracy results
of different methods on different data sets are presented in Table \ref{table_accuracy}. If the data set has more than one new coming feature, we add suffix on the original data set. For example, the data set named Protein1add2 means that the on hand feature is the feature type 1 and the new feature is feature type 2, which have been listed in Table \ref{table_DataDetail}. Together with these results, the paired $t$-test is also conducted. Besides, we also report the results with another popular evaluation metric, i.e., F-score and the results are listed in Table \ref{table_fscore}.

\begin{figure*}[!t]
\centering
\includegraphics[width=0.90\textwidth]{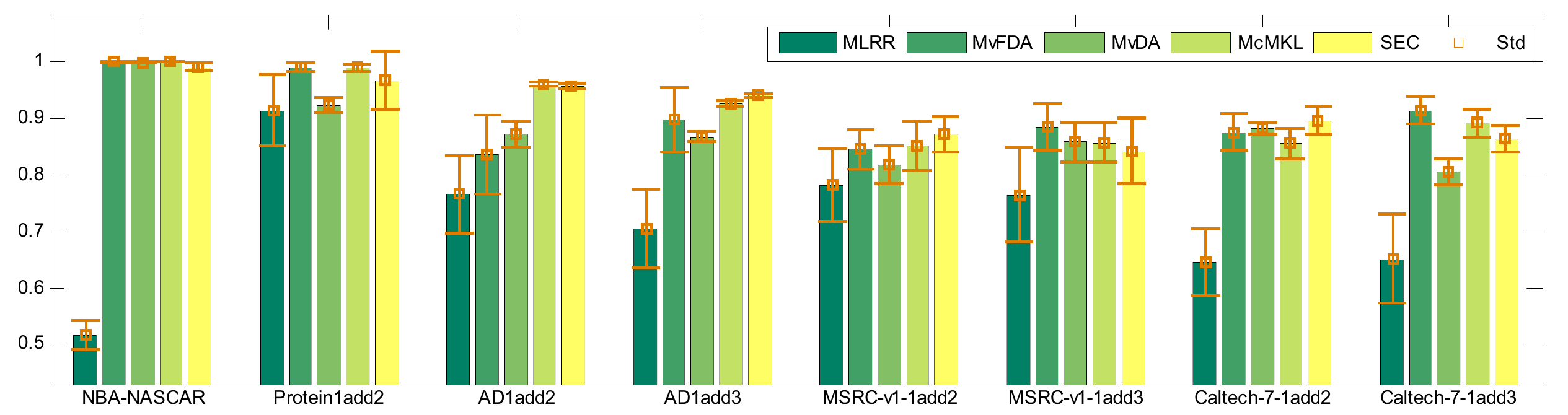}
\caption{Classification accuracies of different multi-view classification methods. Each group corresponds to the results on a data set. The standard deviations are also plotted. The results are 30\% training versus 70\% testing.}
\label{fig3}
\vskip -0.1in
\end{figure*}

\begin{figure*}[!t]
\centering
\includegraphics[width=0.90\textwidth]{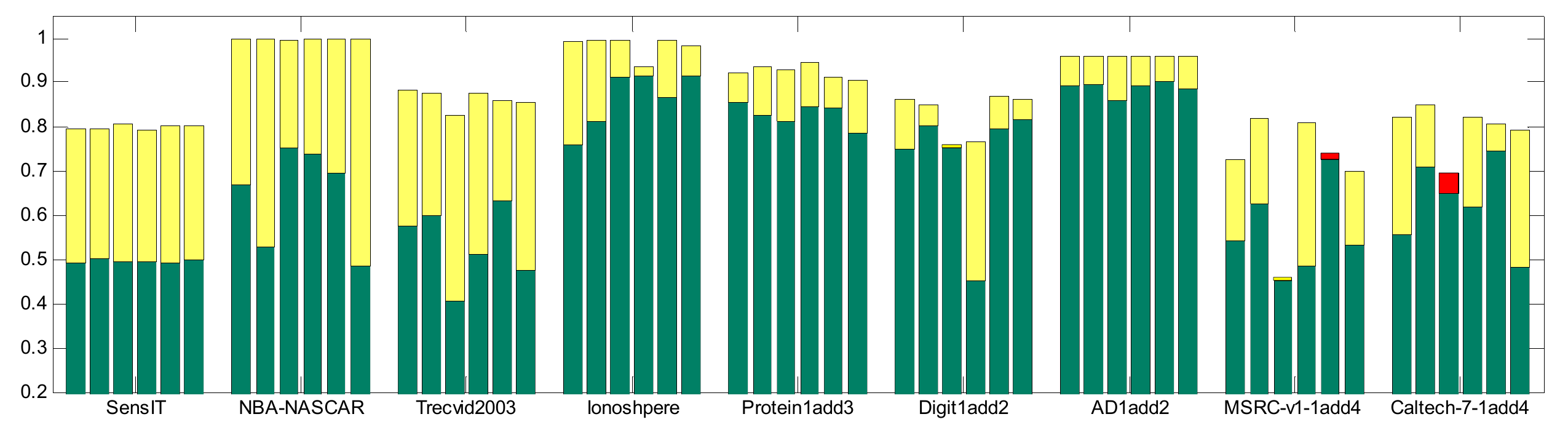}
\caption{The increase of classification accuracies on different data sets. Each group corresponds to the results on a data set. In each group, the original results are plotted by green face and the increasing values are plotted by yellow face if it is positive or red face if it is negative. In each group, from left to right, the methods are Regression, KNN, Naive Bayes, Boosting, Linear SVM and RBF kernel SVM.}
\label{fig4}
\vskip -0.1in
\end{figure*}

There are several observations from the result in Tables \ref{table_accuracy} and \ref{table_fscore}.

(1) The results in Tables \ref{table_accuracy} verify that the security could be guaranteed since SEC outperforms the best classification results of baseline classifiers. The $t$-test results also support this statement.

(2) Compared with the adapted version of six methods, SEC may perform better than all of them in some cases. For example, in the first line of Table \ref{table_accuracy}, SEC achieves the highest accuracies. In most cases, the accuracy of SEC is located between the lowest and highest accuracies of all the adaption methods. The reason may be that the constraints in optimization problem in Eq. (\ref{eq10}) are conducted on all $\bar{\mathbf{Y}}^{(k)}$ since we do not know which one is the best. It will handicap the computing of best $\mathbf{Y}$. In fact, as shown in Proposition \ref{proposition2}, we only need one $\bar{\mathbf{Y}}^{(k)}$ satisfying this constraint.

(3) Although some adaption methods perform better than SEC on some data sets, we cannot find one classifier which always performs better than SEC. This can be seen from the last line of $t$-test results. SEC always wins against the Best single-view results in all data sets whereas other methods may tie or loss on some data sets. This can be explained by the goal of our research, we try to design a secure classifier.

(4) Comparing the results in Tables \ref{table_accuracy} and \ref{table_fscore}, we know that although SEC is designed to guarantee the security in terms of accuracy, it is also secure in terms of another evaluation metric, i.e., F-scores. This may be explained by the fact that accuracy and F-score have some consistency in measuring the classification performance.

\subsection{Comparison With Multi-view Learning Methods}

Our method aims at design a classifier whose performance is never hurt when new features comes. The most direct way in manipulating data with multiple descriptions is using multi-view learning approaches. As illustrated in Table \ref{table_sigle_double_compare}, traditional multi-view learning methods do not always achieve higher accuracy with more features. We study this problem in a more thoroughly way.

We compare SEC with 4 representative multi-view classification approaches belonging to three different categories as mentioned in the Introduction section. McMKL \cite{icml/ZienO07} is the representative method of multi-kernel based approaches. MvFDA \cite{Diethe2008Multiview} and MvDA \cite{pami/KanSZLC16} are two representative subspace learning based methods. MLRR \cite{Zheng2015} is the representative method of regression based approaches. Due to the limitation of space, we report comparison results on 8 data sets with 30\% training vs 70\% testing. The mean and standard deviation of 50 independent runs are plotted in Fig. \ref{fig3} and there are several observations from Fig. \ref{fig3}.

(1) Compared with other multi-view classification approaches, SEC has achieved comparable performances. For example, on the data AD1add3, SEC has the highest accuracy, whereas on Protein1add2, McMKL achieves the best performance.

(2) Among all these methods, it seems that MLRR has the lowest classification accuracies. The reason may be that MLRR has a low rank assumption. The rank should be smaller than the number of classes. In our experiments, the number of class is small and it degrades the performances of MLRR.

(3) None of the multi-view classifiers performs the best consistently on these data sets. For example, MvFDA achieves the highest accuracies on three data sets and MvDA performs best on only one data set. The reason is that different approaches have different assumptions and suit for different data types.

\begin{table*}[!t]
\caption{In the column of $\bar{\mathbf{Y}}^{(k)}$, we denote the number of $\hat{\mathbf{Y}}^{(j)}$ which satisfies the constraint in Eq. (\ref{eq11}). According to Proposition \ref{proposition2}, one each data set, if we have such $\bar{\mathbf{Y}}^{(k)}$, whose corresponding number is 6, SEC is proved to be secure. The rest denotes the distances and accuracies of different methods. Note that, $\min_i \|\mathbf{Y}_i-\mathbf{Y}^{*}\|^2$ is proportional to Best and $\|\mathbf{Y} - \mathbf{Y}^{*} \|^2$ is proportional to SEC.}
\label{table_theorem_verify}
\centering
{
\begin{tabular}{c|| c c c c c c | c c | c c }
\hline
  Data          &  \scriptsize{$\bar{\mathbf{Y}}^{(1)}$}  & \scriptsize{$\bar{\mathbf{Y}}^{(2)}$} & \scriptsize{$\bar{\mathbf{Y}}^{(3)}$} & \scriptsize{$\bar{\mathbf{Y}}^{(4)}$} & \scriptsize{$\bar{\mathbf{Y}}^{(5)}$} & \scriptsize{$\bar{\mathbf{Y}}^{(6)}$} & \scriptsize{$\min_i \|\mathbf{Y}_i-\mathbf{Y}^{*}\|^2$} & \scriptsize{$\| \mathbf{Y} - \mathbf{Y}^{*} \|^2$} & Best & SEC \\
\hline
NBA-NASCAR      &    6    &   5  &  6  &  6  &  6  &  6  & 228.0000 & 2.0000   &  0.7286 & 0.9976  \\
Ionoshpere      &    6    &   6  &  6  &  6  &  6  &  6  & 20.0000  & 0.0000   &  0.9429 & 1.0000  \\
MSRC-v1-1add4   &    0    &   0  &  6  &  0  &  0  &  0  & 60.0000  & 26.2255  &  0.7143 & 0.8476  \\
MSRC-v1-2add4   &    0    &   0  &  5  &  1  &  1  &  1  & 8.0000   & 26.0674  &  0.9619 & 0.8667   \\
Caltech-7-1add4 &    0    &   1  &  6  &  0  &  0  &  0  & 114.0000 & 48.5492  &  0.7409 & 0.8636  \\
Caltech-7-2add4 &    1    &   0  &  4  &  0  &  0  &  0  & 38.0000  & 51.4603  &  0.9136 & 0.8591 \\
\hline
\end{tabular}}
\vskip -0.1in
\end{table*}

\subsection{The Effectiveness of Adaption}

In above experiments, we only show the best results of baseline classifiers. Since adaption is an important step to ensure security, we would like to compare the results before and after the adaption. With the training and testing percentages as 50\% vs 50\%, we report results on 9 data sets shown in Fig. \ref{fig4}. We use the stacked bar to show the results. In each group, from left to right, the methods, which are plotted by green bar, are Regression, KNN, Naive Bayes, Boosting, Linear SVM and RBF kernel SVM. If the adaption increases the accuracy, we add a yellow bar on the original one. On the contrary, if the adaption decreases the accuracy, we stack a red bar on the original one. The other settings are the same as that in Table \ref{table_accuracy} and we have the follows observations.

(1) The adaption improves the performance of original methods in most cases, no matter which type of classifier and which kind of data have been employed. Nevertheless, the adaption may also worsen the performances. For example, on MSRC-v1-1add4 the adaption decreases the accuracy of Linear SVM. The reason may be that the new coming features have distinct characters, compared with the on hand features. This phenomenon has also demonstrated that there is no strategy which can always improve the performances of original methods.

(2) Although we employ the same kind of adaption strategy on all methods, it takes different effects on different data sets. For example, on Trecvid2003, the improvement is significant whereas on Protein1add3, the improvement is limited. This may be caused by the reason that the new coming features of Trecvid2003 are suitable to this kind of adaption.

(3) Even on the same data set, the improvement of different methods varies. For example, on Digit1add2, the improvement of Naive Bayes is the largest. One possible reason is that Naive Bayes performs not so well on this data set and this kind of adaption will heavily enlarge the effects of new coming features.

\subsection{Demonstration of the Condition in Proposition 2}

Eq. (\ref{eq11}) in Proposition \ref{proposition2} is a sufficient condition to guarantee the security. We will verify this statement through numerical results. We conduct experiments on six data sets, shown in Table \ref{table_theorem_verify}, with the same setting as that in previous subsection. In the column $\bar{\mathbf{Y}}^{(k)}$ of Table \ref{table_theorem_verify}, we denote the number of $\hat{\mathbf{Y}}^{(j)}$, which satisfies the constraint in Eq. (\ref{eq11}). In the columns of $\min_i \|\mathbf{Y}_i-\mathbf{Y}^{*}\|^2$ and $\|\mathbf{Y} - \mathbf{Y}^{*} \|^2$, we record the distance. Finally, the last two columns consist of classification accuracies of Best and SEC.

As seen from the results in Table \ref{table_theorem_verify}, we have

(1) The results in Table \ref{table_theorem_verify} verifies the correctness of Proposition \ref{proposition2}. On the first three data sets, we have at least one $\bar{\mathbf{Y}}^{(k)}$, all the $\hat{\mathbf{Y}}^{(j)}$ satisfy the constraint in Eq. (\ref{eq11}). Thus, the distance $\|\mathbf{Y} - \mathbf{Y}^{*} \|^2$ is smaller than $\min_i \|\mathbf{Y}_i-\mathbf{Y}^{*}\|^2$ and the classifier accuracy of SEC is higher than Best.

(2) When there is no such $\bar{\mathbf{Y}}^{(k)}$, whose corresponding number is 6, our method cannot guarantee secure, this phenomenon can be seen from the results on MSRC-v1-2add4 and Caltech-7-2add4. These results also demonstrate that the adding of new feature is not always helpful.

(3) The security is feature type dependent. For example, on MSRC-v1, if the on hand feature type is 1, the adding of feature with type 4 is helpful. Nevertheless, if the original feature is type is 2, the adding of the same type of features degrades the performance.

\subsection{Diagnostic Classification of Schizophrenia}
\label{sec_application}

Schizophrenia is one of severe mental illnesses that impairs multiple cognitive domains as evidenced by delusions, hallucinations, disorganized speech and thought formation, social withdrawal, gross disorganization, and other negative symptoms \cite{APA2013}. The complex and heterogeneous symptoms pose a challenge to the objective diagnosis of schizophrenia based solely on clinical manifestations. Searching for reliable biomarkers for the diagnosis and treatment of schizophrenia is clearly an international imperative. Our previous works involved the use of MRI to discriminate schizophrenic patients from healthy controls, implying the potential of MRI in the diagnosis of schizophrenia \cite{Shen2009, Wang2015Evidence}.

In this section, we will tackle this problem from another perspective, i.e., security of using augmented features. There are two public data sets are employed for evaluation. The COBRE data\footnote{http://fcon\_1000.projects.nitrc.org/indi/retro/cobre.html} consists of raw 3D-T1 and resting-state blood oxygenation level-dependent (BOLD)-fMRI brain images from 71 patients with schizophrenia and 74 healthy controls. Diagnostic information was collected using the Structured Clinical Interview used for DSM Disorders (SCID). The UCLA data \cite{Poldrack2016A} consists of raw 3D-T1 and BOLD-fMRI data from 58 patients with schizophrenia and 132 healthy controls.

\begin{figure}[!t]
\centering
\includegraphics[width=0.45\textwidth]{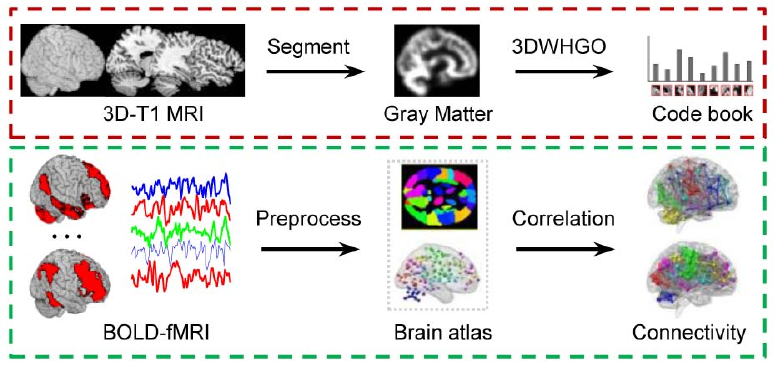}
\caption{The sample images of fMRI and the illustrations of our processing in extract the descriptions. The half top are illustrations of extraction from 3D-T1 MRI images and the bottom are extraction from BOLD fMRI.}
\label{fig5}
\vskip -0.1in
\end{figure}

\begin{table*}[!t]
\caption{Testing accuracies (mean$\pm$std) of the compared methods on fMRI data sets with different percent of training and testing examples. The other settings are the same as that in Table \ref{table_accuracy}.}
\label{table_frmi_acc}
\centering
\vskip -0.1in
{
\begin{tabular}{c||c|| c c c c c c c c }
\hline
COBRE& Per       &  Best   &  AdRegression   & AdKNN    & AdNaiveBayes    & AdBoosting &  AdSVM(Lin) & AdSVM(RBF) & SEC \\
\hline
1add2& 0.2 & .5418(.0453) & .6591(.0684)$\bullet$ & .6698(.0516)$\bullet$ & .5049(.0286)$\circ$ & .6315(.0778)$\bullet$ & .6621(.0581)$\bullet$ & .6819(.0520)$\bullet$ & .6250(.0916)$\bullet$ \\
&0.3 & .5668(.0535) & .7059(.0482)$\bullet$ & .7000(.0474)$\bullet$ & .5213(.0531)$\circ$ & .6817(.0461)$\bullet$ & .6941(.0426)$\bullet$ & .7005(.0414)$\bullet$ & .6911(.0803)$\bullet$ \\
&0.4 & .5690(.0413) & .7201(.0538)$\bullet$ & .6937(.0537)$\bullet$ & .5351(.0666){\tiny $\odot$} & .6856(.0472)$\bullet$ & .6960(.0488)$\bullet$ & .7006(.0351)$\bullet$ & .7017(.0931)$\bullet$ \\
&0.5 & .5826(.0497) & .7271(.0623)$\bullet$ & .6958(.0594)$\bullet$ & .5722(.0722){\tiny $\odot$} & .6972(.0777)$\bullet$ & .7049(.0529)$\bullet$ & .7257(.0527)$\bullet$ & .7118(.0674)$\bullet$ \\
&0.6 & .5940(.0717) & .7276(.0446)$\bullet$ & .7164(.0463)$\bullet$& .5853(.0960){\tiny $\odot$} & .7224(.0630)$\bullet$ & .7259(.0518)$\bullet$ & .7371(.0509)$\bullet$ & .7259(.0512)$\bullet$ \\
\hline
1add3& 0.2 & .5483(.0849) & .6069(.0883){\tiny $\odot$} & .6190(.0705){\tiny $\odot$} & .5088(.0244)$\circ$ & .5914(.0845){\tiny $\odot$} & .5966(.0904){\tiny $\odot$} & .6207(.0716){\tiny $\odot$} & .6034(.0917){\tiny $\odot$} \\
     & 0.3 & .5693(.0627) & .7079(.0361)$\bullet$ & .6970(.0370)$\bullet$ & .5248(.0625)$\circ$ & .6822(.0376)$\bullet$ & .6970(.0328)$\bullet$ & .7233(.0306)$\bullet$ & .6906(.0689)$\bullet$ \\
     & 0.4 & .6057(.0427) & .7345(.0389)$\bullet$ & .6989(.0573)$\bullet$ & .5632(.0568){\tiny $\odot$} & .7103(.0549)$\bullet$ & .6908(.0425)$\bullet$ & .7299(.0364)$\bullet$ & .7207(.0420)$\bullet$ \\
     & 0.5 & .5917(.0322) & .7389(.0268)$\bullet$ & .7000(.0354)$\bullet$ & .5528(.0515){\tiny $\odot$} & .7278(.0516)$\bullet$ & .7028(.0611)$\bullet$ & .7097(.0384)$\bullet$ & .7294(.0547)$\bullet$ \\
     & 0.6 & .5931(.0461) & .7431(.0611)$\bullet$ & .7241(.0430)$\bullet$ & .5759(.0973){\tiny $\odot$} & .7138(.0636)$\bullet$ & .7017(.0598)$\bullet$ & .7466(.0586)$\bullet$ & .7310(.0737)$\bullet$ \\
\hline
UCLA & Per       &  Best   &  AdRegression   & AdKNN    & AdNaiveBayes    & AdBoosting &  AdSVM(Lin) & AdSVM(RBF) & SEC \\
\hline
1add2& 0.2 &.7243(.0354) & .7500(.0526){\tiny $\odot$} & .7421(.0321){\tiny $\odot$} & .5862(.2031)$\circ$ & .7283(.0446){\tiny $\odot$} & .7224(.0496){\tiny $\odot$} & .7112(.0459){\tiny $\odot$} & .7531(.0482)$\bullet$ \\
     & 0.3 &.7338(.0490) & .7571(.0468){\tiny $\odot$} & .7564(.0362){\tiny $\odot$} & .6977(.0292){\tiny $\odot$} & .7429(.0340){\tiny $\odot$} & .6835(.0531)$\circ$ & .6827(.0673)$\circ$ & .7559(.0365){\tiny $\odot$} \\
     & 0.4 &.7491(.0351) & .7816(.0342){\tiny $\odot$} & .7667(.0323){\tiny $\odot$} & .7079(.0442)$\circ$ & .7737(.0532){\tiny $\odot$} & .7184(.0711){\tiny $\odot$} & .7044(.0662){\tiny $\odot$} & .7574(.0634){\tiny $\odot$} \\
     & 0.5 &.7305(.0365) & .7611(.0371){\tiny $\odot$} & .7379(.0235){\tiny $\odot$} & .6842(.0302)$\circ$ & .7579(.0337){\tiny $\odot$} & .6842(.0700){\tiny $\odot$} & .6589(.0793)$\circ$ & .7689(.0291)$\bullet$ \\
     & 0.6 &.7474(.0499) & .8000(.0492)$\bullet$  & .7618(.0418){\tiny $\odot$} & .7105(.0588){\tiny $\odot$} & .7855(.0468){\tiny $\odot$} & .6750(.0206)$\circ$ & .6750(.0537)$\circ$ & .7934(.0565)$\bullet$ \\
\hline
1add3 &0.2 &.7125(.0234) & .7145(.0321){\tiny $\odot$} & .7329(.0278){\tiny $\odot$} & .5191(.0894)$\circ$ & .7250(.0282){\tiny $\odot$} & .7033(.0332){\tiny $\odot$} & .6796(.0318)$\circ$ & .7476(.0321)$\bullet$ \\
      &0.3 &.7398(.0337) & .7692(.0444){\tiny $\odot$} & .7624(.0254){\tiny $\odot$} & .6683(.0853)$\circ$ & .7707(.0412){\tiny $\odot$} & .7120(.0450){\tiny $\odot$} & .7008(.0531){\tiny $\odot$} & .7519(.0462){\tiny $\odot$} \\
      &0.4 &.7237(.0370) & .7456(.0620){\tiny $\odot$} & .7518(.0424){\tiny $\odot$} & .6784(.0807)$\circ$  & .7368(.0510){\tiny $\odot$} & .6930(.0455){\tiny $\odot$} & .7088(.0376){\tiny $\odot$} & .7565(.0603){\tiny $\odot$} \\
      &0.5 &.7126(.0371) & .7484(.0728){\tiny $\odot$} & .7411(.0695){\tiny $\odot$} & .6726(.0316)$\circ$  & .7516(.0817){\tiny $\odot$} & .7116(.0581){\tiny $\odot$} & .6874(.0645){\tiny $\odot$} & .7600(.0679){\tiny $\odot$} \\
      &0.6 &.7776(.0300) & .8289(.0340)$\bullet$ & .7829(.0272){\tiny $\odot$} & .7263(.0438)$\circ$  & .7921(.0424){\tiny $\odot$} & .7316(.0729){\tiny $\odot$} & .7276(.0544)$\circ$  & .8182(.0422)$\bullet$ \\
\hline
\multicolumn{3}{c}{Win/Tie/Loss}&    11/9/0  & 9/11/0 & 0/8/12 & 9/11/0 & 9/9/2 & 9/6/5 & 14/6/0  \\
\hline
\end{tabular}}
\vskip -0.0in
\end{table*}

Before going into the details, we would like to give an intuitive illustration in Fig. \ref{fig5}. As seen from the top plane of Fig. \ref{fig5}, different from traditional images, brain MRI images are often three-dimensional. Thus, we employ a three-dimensional feature descriptor in our previous work \cite{Yuan2015Gender}, i.e., the three-dimensional weighted histogram of gradient orientation (3D WHGO) to describe this complex spatial structure of 3D-T1 brain images.  The descriptor combines local information for signal intensity and global three-dimensional spatial information for the whole brain. This is the on hand descriptions (feature 1) and the dimensionality is 9600. Besides, as shown in the bottom plane of Fig. \ref{fig5}, similar to our previous work \cite{brain/zeng}, we use BOLD-fMRI scans to extract region-to-region functional connectivity features with 116 regions of interest (ROI) from an automated anatomical labeling (AAL) brain atlas \cite{TZOURIOMAZOYER2002273} and 160 ROIs (modeled as 6-mm radius spheres) from several meta-analyses of fMRI activation studies \cite{Dosenbach2010Prediction}, respectively. They are the first (feature 2) and the second (feature 3) new coming features and the number of features are 6670 and 12720 respectively.

Similarly, we compare SEC with other algorithms as shown in Table \ref{table_accuracy}. The other settings are the same as that in the experiments in Section 4.2. With different percent of training data points, we summary the results in Table \ref{table_frmi_acc} and have the similar observations as that in Table \ref{table_accuracy}. Significantly, even when all the adaption methods perform better than Best slightly, SEC still guarantees security. It can be verified by the $t$-test results in the last line. Besides, compared with the results on COBRE, the results on UCLA indicate that the adaption plays a less important role. The adapted accuracies may be even worse. Nevertheless, in these serious cases, SEC also achieves satisfying results.

\section{Conclusion}
\label{sec_con}

In this paper, we study the problem of learning a secure classifier, and propose the SEC algorithm that does not need to access the on hand features if the classifiers are trained. Our approach is particularly useful in the applications with emerging new features, where robust learning performances are needed. In practice, our methods benefit SEC in two aspects: (1) In A-stage, we can adapt any types of classifiers by utilizing new coming features. It is quite necessary in practical applications; (2) In I-stage, we integrate the benefits of the adapted classifier from A-stage, which is also very useful for real applications. Extensive empirical studies show that the SEC approach is very effective to solve the secure classifier learning problem. In this paper, we only focus on the adding of one type of new coming features. How to extend it into multiple types of new coming features is an interesting future work. A possible way may be that we can accept them by adaption and integration one by one. Besides, how to accelerate the optimization speed is also worth studying. Several modern optimization tools should be borrowed for alleviating computational burden.

\bibliographystyle{IEEEtran}
\bibliography{sec}

\end{document}